\documentclass{article}

\usepackage{microtype}
\usepackage{graphicx}
\usepackage{subfigure}
\usepackage{booktabs} 

\usepackage{hyperref}
\usepackage{multirow}
\newcommand\Tstrut{\rule{0pt}{2.2ex}}

\usepackage[accepted]{icml2023}

\usepackage{amsmath}
\usepackage{amssymb}
\usepackage{mathtools}
\usepackage{amsthm}

\usepackage[capitalize,noabbrev]{cleveref}

\theoremstyle{plain}
\newtheorem{theorem}{Theorem}[section]

\newtheorem{corollary}[theorem]{Corollary}
\theoremstyle{definition}
\newtheorem{definition}[theorem]{Definition}

\theoremstyle{remark}

\usepackage[textsize=tiny]{todonotes}

\icmltitlerunning{Learnability and Algorithm for Continual Learning}

\begin{document}

\twocolumn[

\icmltitle{Learnability and Algorithm for Continual Learning}



\icmlsetsymbol{equal}{*}

\begin{icmlauthorlist}
\icmlauthor{Gyuhak Kim}{equal,sch}
\icmlauthor{Changnan Xiao}{equal,comp1}
\icmlauthor{Tatsuya Konishi}{comp2}
\icmlauthor{Bing Liu}{sch}
\end{icmlauthorlist}

\icmlaffiliation{sch}{Department of Computer Science, University of Illinois at Chicago.}
\icmlaffiliation{comp1}{Work done at ByteDance.}
\icmlaffiliation{comp2}{KDDI Research (work done when this author was visiting Bing Liu's group)}

\icmlcorrespondingauthor{Bing Liu}{liub@uic.edu}

\icmlkeywords{Machine Learning, ICML}

\vskip 0.3in
]



\printAffiliationsAndNotice{\icmlEqualContribution} 

\begin{abstract}
This paper studies the challenging \textit{continual learning} (CL) setting of \textit{Class Incremental Learning} (CIL). CIL learns a sequence of tasks consisting of disjoint sets of concepts or classes. At any time, a single model is built that can be applied to predict/classify test instances of any classes learned thus far without providing any task related information for each test instance. Although many techniques have been proposed for CIL, they are mostly empirical. It has been shown recently that a strong CIL system needs a strong within-task prediction (WP) and a strong out-of-distribution (OOD) detection  for each task. However, it is still not known whether CIL is actually learnable. {This paper shows that CIL is learnable.} Based on the theory, a new CIL algorithm is also proposed. Experimental results demonstrate its effectiveness.
\end{abstract}

\section{Introduction}
\label{sec.intro}

Learning a sequence of tasks incrementally, called \textit{continual learning}, has attracted a great deal of attention recently~\citep{chen2018lifelong}. In the supervised learning context, each task consists of a set of concepts or classes to be learned. 
It is assumed that all tasks are learned in one neural network, which results in the key challenge of \textit{catastrophic forgetting} (CF) because when learning a new task, the system has to modify the network parameters learned from old tasks in order to learn the new task, which may cause performance degradation for the old tasks~\citep{McCloskey1989}. 
Two continual learning settings have been popularly studied: \textit{task incremental learning} (TIL)~\citep{van2019three} and \textit{class incremental learning} (CIL). 

In TIL, each task is an independent classification problem and has a separate model (the tasks may overlap). At test time, the task-id of each test instance is provided to locate the task-specific model to classify the test instance. 
\begin{definition}[Task Incremental Learning (TIL)]
\label{def:til}
    TIL learns a sequence of tasks, $1, 2, ..., T$. The training set of task $k$ is 
    $\mathcal{D}_k=\{((x^i_k, k), y_k^i)_{i=1}^{n_k}\}$, 
    where $n_k$ is the number of samples in task $k \in \mathcal{T} = \{1, 2, ..., T\}$, and $x^i_k \in \mathcal{X}$ is an input sample and $y^i_k \in {Y}_k \subseteq \mathcal{Y}$ ($=\bigcup_{k=1}^T {Y}_k$) is its label.
    A TIL system learns a function $f: \mathcal{X} \times \mathcal{T} \rightarrow \mathcal{Y}$ to assign a class label $y \in {Y}_k$ to $(x, k)$ (a test instance $x$ from task $k$). 
\end{definition}

For CIL, a single model is built for all tasks/classes  
learned thus far (the classes in each task are distinct). At test time, no task-id is provided for a test instance. 

\begin{definition}[Class Incremental Learning (CIL)] \label{def:cil}
    CIL learns a sequence of tasks, $1, 2, ..., T$. The training set of task $k$ is 
    $\mathcal{D}_k=\{(x^i_k, y_k^i)_{i=1}^{n_k}\}$, 
    where $n_k$ is the number of samples in task $k \in \mathcal{T} = \{1, 2, ..., T\}$, and $x^i_k \in \mathcal{X}$ is an input sample and $y^i_k \in {Y}_k \subset \mathcal{Y}$ ($=\bigcup_{k=1}^T {Y}_k$) is its label. 
    All ${Y}_k$'s are disjoint (${Y}_k \cap {Y}_{k'} = \emptyset,\, \forall k \neq k'$). 
    A CIL system learns a function (predictor or classifier) $f : \mathcal{X} \rightarrow \mathcal{Y}$ that assigns a class label $y$ to a test instance $x$. 
\end{definition}
CIL is a more challenging setting because in addition to CF, it has the \textit{inter-task class separation} (ICS)
\cite{kim2022theoretical} problem. ICS refers to the situation that since the learner has no access to the training data of the old tasks when learning a new task, then the learner has no way to establish the decision boundaries between the classes of the old tasks and the classes of the new task, which results in poor classification accuracy. 
\citet{kim2022theoretical} showed
that a good within-task prediction (WP) and a good \textit{out-of-distribution} (OOD) detection for each task are \textit{necessary} and \textit{sufficient} conditions for a strong CIL model. 
\begin{definition}[out-of-distribution (OOD) detection] \label{def:ood}
    Given the training set
    $\mathcal{D}=\{(x^i, y^i)_{i=1}^{n}\}$, where $n$ is the number of data samples, $x^i \in \mathcal{X}$ is an input sample and $y^i \in \mathcal{Y}$ is its class label. $\mathcal{Y}$ is the set of all classes in $\mathcal{D}$ and called the \textit{in-distribution} (IND) classes. Our goal is to learn a function $f : \mathcal{X} \rightarrow \mathcal{Y} \cup \{O\}$ that can detect test instances that do not belong to any classes in $\mathcal{Y}$  (OOD)), which are assigned to class $O$, denoting \textit{out-of-distribution} (OOD). 
\end{definition}

The intuition of the theory in  \cite{kim2022theoretical} is that if OOD detection is perfect for each task, then a test instance will be assigned to the correct task model to which the test instance belongs for classification, i.e., within-task prediction (WP). However, \cite{kim2022theoretical} does not prove that CIL is learnable. To our knowledge, no existing work has reported a learnability study for CIL (see Sec.~\ref{sec.related}). This paper performs the \textbf{CIL learnability} study. 

The proposed learnability proof requires two assumptions: \textbf{(1)} OOD detection is learnable. Fortunately, this has been proven in a recent paper~\cite{fang2022out}. \textbf{(2)} There is a mechanism that can completely overcome forgetting (CF) for the model of each task. Again, fortunately, there are many existing TIL methods that can eliminate forgetting, e.g., { parameter-isolation methods such as HAT~\cite{Serra2018overcoming} and SupSup~\cite{supsup2020}, which work by learning a sub-network in a shared network for each task. The sub-networks of all old tasks are protected when training a new task. Orthogonal projection methods such as PCP~\cite{kim2020continual} and CUBER~\cite{lin2022beyond}} can also overcome forgetting in the TIL setting. 

CIL can be solved by a combination of \textbf{[a]} a \textit{TIL method} that is able to protect each task model with no CF, and \textbf{[b]} a \textit{normal supervised learning method for WP} and \textbf{[c]} an \textit{OOD detection} method. \textbf{[b]} and \textbf{[c]} can be easily combined either \textbf{(i)} with an OOD detection model since it also learns the IND classes (see \textbf{Definition}~\ref{def:ood}) or \textbf{(ii)} a WP model that can also perform OOD detection. That is, for CIL, we simply replace the classification model built for each task in HAT/SupSup with a combined WP and OOD detection model. 

Based on the theory, we propose a new replay-based CIL method that uses the combination of \textbf{[a]} and \textbf{(ii)} (two separate heads for each task, one for WP and the other for OOD detection based on  the same feature extractor). 
This paper thus makes two main contributions:

(1). It performs \textbf{the first learnability study of CIL}. To the best of knowledge, no such a study has been reported so far. 

(2). Based on the theory, \textbf{a new CIL method}, called ROW (\textit{Replay, OOD, and WP} for CIL), is proposed. Experimental results show that it outperforms existing strong baselines. 

{It is interesting to note that our theory, including our earlier work in~\cite{kim2022theoretical}, in fact, unifies OOD detection and continual learning as it covers both~\cite{kim2023open}. Additionally, the theory is also applicable to \textbf{\textit{open world learning}}  because OOD detection and class incremental learning are two critical components of an open world learning system~\cite{liu2023ai}. }

\section{Related Work}
\label{sec.related}
To our knowledge, we are not aware of any paper that studies the learnability of CIL. Below, we survey the existing CL literature on both the theoretical and empirical sides. 

On the \textbf{theoretical side}. \citet{pentina2014pac} proposes a PAC-Bayesian framework to provide a learning bound on expected error by the average loss on the observed tasks. However, this work is not about CIL but about TIL. It focuses on knowledge transfer and assumes that all the tasks have the same input space and the same label space and the tasks are very similar. However, in CIL, every task has a distinct set of class labels. Furthermore, this work is not concerned with CF as earlier research in lifelong learning builds a separate model for each task. 
\citet{lee2021continual} studied the generalization error by task similarity. It is again about TIL. 
\citet{bennani2020generalisation} showed that a specific method called orthogonal gradient descent (OGD) gives a tighter generalization bound than SGD. 
As noted in Sec.~\ref{sec.intro}, empirically, the CF problem for TIL has been solved~\cite{Serra2018overcoming,kim2022theoretical}. Several techniques have also been proposed to carry out knowledge transfer \cite{ke2020continual,ke2021achieving,lin2022beyond}. 
Our work is entirely different as we study the learnability of CIL, which is a more challenging setting than TIL because of the additional difficulty of ICS~\cite{kim2022theoretical} in CIL. In this work, we are not concerned with knowledge transfer, which is mainly studied for the TIL setting. 
Recently, \citet{kim2022theoretical} showed that a good within-task prediction (WP) and a good OOD detection for each task are necessary and sufficient conditions for a strong CIL model. 
However, \citet{kim2022theoretical} did not show that CIL is learnable. This paper performs this study. It also proposes a new CIL algorithm. 

On the \textbf{empirical side}, a large number of algorithms have been proposed. They belong to several families. 
\textbf{(1).} {\textit{Regularization-based methods}} mitigate CF by restricting the learned parameters for old tasks from being updated significantly in a new task learning using regularizations \citep{Kirkpatrick2017overcoming,Zenke2017continual} or knowledge distillation  \citep{Li2016LwF,Zhu_2021_CVPR_pass}. 
Many existing approaches have used similar approaches \citep{ritter2018online,schwarz2018progress,xu2018reinforced,castro2018end,Dhar2019CVPR,lee2019overcoming,ahn2019neurIPS,Liu2020}. \textbf{(2).} {\textit{Replay-based methods}} alleviate CF by saving a small amount of training data from old tasks in a memory buffer and jointly train the model using the current data and the saved buffer data \cite{Rusu2016,Lopez2017gradient,Chaudhry2019ICLR,hou2019learning,wu2019large,rolnick2019neurIPS, NEURIPS2020_b704ea2c_derpp,rajasegaran2020adaptive,Liu2020AANets,Cha_2021_ICCV_co2l,yan2021dynamically,wang2022memory,guo2022online,kim2022multi}. Some methods in this family also study which samples in memory should be used in replay~\citep{aljundi2019online} or which samples in the training data should be saved \citep{Rebuffi2017,liu2020mnemonics}. \textbf{(3).} \textit{Pseudo-replay methods} generate pseudo replay data for old tasks to serve as the replay data~\citep{Kamra2017deep,Shin2017continual,wu2018memory,Seff2017continual,Kemker2018fearnet,hu2019overcoming,Rostami2019ijcai,ostapenko2019learning}. \cite{Zhu_2021_CVPR_pass} generates features instead of raw data. 
\textbf{(4).} \textit{Parameter-isolation methods} train and protect a sub-network for each task \citep{Mallya2017packnet,abati2020conditional,von2019continual_hypernet,rajasegaran2020itaml,hung2019compact,henning2021posterior}. Several systems, e.g., HAT \cite{Serra2018overcoming} and SupSup~\cite{supsup2020}, have largely eliminated CF. A limitation is that the task-id of each test instance must be provided. These methods are thus mainly used for TIL. \textbf{(5).} \textit{Orthogonal projection} methods learn each task in an orthogonal space to other tasks to reduce task interference or CF 
\cite{zeng2019continuous,kim2020continual,chaudhry2020continual,lin2022beyond}. 

Our empirical part of the work is related to but also very different from the above methods. We use the replay data as OOD training data to fine-tune an OOD detection head for each task based on the features learned for the WP head and uses the TIL method HAT to overcome CF. Some existing methods have used a TIL method for CIL with an additional task-id prediction technique. iTAML~\cite{rajasegaran2020itaml}'s task-id prediction needs the test data come in batches and each batch must be from the same task, which is unrealistic as the test sample usually comes one after another. CCG~\cite{abati2020conditional}, Expert Gate~\cite{Aljundi2016expert}, HyperNet~\cite{von2019continual_hypernet} and PR-Ent~\cite{henning2021posterior} either build a separate network or use entropy to predict the task-id. LMC~\cite{ostapenko2021continual} learns task specific knowledge via local modules capable of task-id prediction. However, they all perform poorly because none of the systems deal with the ICS problem, which is the key and is what our OOD detection is trying to tackle. In this line of work, the most closely related work to ours is MORE~\cite{kim2022multi}, which builds a model for each task treating the replay data as OOD data. However, in inference, it considers only the IND classes of each task, but not OOD detector. Our method is more principled and outperforms MORE. The methods in~\cite{kim2022theoretical} do not use replay data and perform worse.

\section{Learnability of CIL} \label{sec.learnability}
Before going to the learnability analysis, we first describe the intuition behind. \citet{kim2022theoretical} showed that given a test sample, the CIL prediction probability for each class in a task is the product of two prediction probabilities: \textit{within-task prediction} (WP) and \textit{task-id prediction} (TP), 
\begin{align}
     \mathbf{P}(X_{k,j} | x)
     &= \mathbf{P} (X_{k,j} | x, k) \mathbf{P}(X_{k} | x), \label{eq:cil_in_til_and_tp}
\end{align}
where $X_{k,j}$ is the domain of task $k$ and class $j$ of the task and $x$ is an instance. The first probability on the right-hand-side (RHS) is WP and the second probability on the RHS is TP. 
However, as mentioned earlier, \citet{kim2022theoretical} did not study whether CIL is learnable. 
We also note that \cite{kim2022theoretical} proved that TP and OOD are correlated and only differ by a constant factor. Based on the definition of OOD detection (\textbf{Definition}~\ref{def:ood}), the OOD detection model can also perform WP. In the recent work in \cite{fang2022out}, it is proven that OOD detection is learnable. 

We show that if the learning of each task does not cause catastrophic forgetting (CF) for previous tasks, then CIL is learnable.\footnote{The current learnability result only applies to offline CIL, but not to online CL, where the task boundary may be blurry.} 
Fortunately, {CF can be prevented for each task}
as several existing \textit{task incremental learning} (TIL) methods 
including but not limited to HAT~\cite{Serra2018overcoming} and SupSup~\cite{supsup2020} in the parameter-isolation family and PCP~\cite{kim2020continual} and CUBER~\cite{lin2022beyond} in the orthogonal projection family can ensure no CF \cite{kim2022theoretical}. 
HAT/SupSup basically trains a sub-network as the model for each task. In learning each new task, all the sub-networks for the previous tasks are protected with masks so that their parameters will not be modified in backpropagation. Thus, in this section, we assume that all tasks are learned without \textit{catastrophic forgetting} (CF).

\begin{table}
\vspace{-2mm}
\caption{
{Notations used in Sec.~\ref{sec.learnability}.
}
}
\label{Tab:notation_thm}
\centering
\resizebox{1.0\columnwidth}{!}{
\begin{tabular}{c c}
\toprule
\multicolumn{1}{c}{Notation} & \multicolumn{1}{c}{Description}\\
\midrule
$\mathcal{X}$ & feature space\\
$\mathcal{Y}$ & label space \\
$\mathcal{H}$ & hypothesis space \\
$X_k$ & random variable in $\mathcal{X}$ of task $k$ \\
$Y_k$ & random variable in $\mathcal{Y}$ of task $k$ \\
$D_{(X_k, Y_k)}$ & distribution of task $k$ \\
$l$ & loss function \\
$h$ & hypothesis function in $\mathcal{H}$ \\
$\mathbf{R}_{D_{(X, Y)}}$ & risk function, expectation of loss function on $D_{(X, Y)}$\\
$\mathcal{D}$ & set of all distributions \\ 
$S$ & set of samples \\
$D_{[1:k]}$ & weighted mixture of the first $k$ distributions\\
$\pi_k$ & mixture weight \\
$D_{k}$ & equivalent to $D_{[k:k]}$ and $D_{(X_k, Y_k)}$\\
$S|_{[k1:k2]}$ & set of support samples for $D_{[k1:k2]}$ \\
$S|_{k}$ & equivalent to $S|_{[k:k]}$ \\
$\mathbf{A}_k$ & algorithm after training the $k$-th task \\
$z_{k}^{i,j}$ & score function of the $j$-th class of the $i$-th task for task $k$ 
\\
$O$ & distribution of OOD \\
$\alpha$ & constant in $[0, 1)$ \\
$D^{\alpha}$ & mixture of $D$ and $O$ with weight $\alpha$ \\
$z_{k}^{o}$ & score for the OOD class \\
$\emptyset$ & empty set \\
$\epsilon_n$ & error rate with total number of samples $n$ \\

\bottomrule
\end{tabular}
}
\vspace{-3mm}
\end{table}

We now discuss the learnability of class incremental learning (CIL). The notations for the following discussion are described in Tab.~\ref{Tab:notation_thm}.
Let $\mathcal{X}$ be a feature space, $\mathcal{Y}$ a label space, and  $\mathcal{H}$ a hypothesis function space. Assume $\mathcal{H}$ is a \textit{ring}, because we construct hypothesis functions by addition and multiplication in the proof of \textbf{Theorem} \ref{thm:def2_not_imply_def1} and \textbf{Theorem} \ref{thm:def4_imply_def3}.
We use $D_{(X_k, Y_k)}$ to denote the distribution of task $k$. $X_k \in \mathcal{X}$ and $Y_k \in \mathcal{Y}$ are random variables following $D_{(X_k,  Y_k)}$. $\mathcal{D} = \{D_{(X, Y)}\}$ denotes the set of all distributions. $l(y_1, y_2) \geq 0$ denote a loss function. Denote $h \in \mathcal{H}$ as a hypothesis function. For any $X \in \mathcal{X},\, Y \in \mathcal{Y}$, the risk function $\mathbf{R}_{D_{(X, Y)}}(h) \overset{def}{=} \mathbb{E}_{(x, y) \sim D_{(X, Y)}}[l(h(x), y)]$. $S \overset{def}{=} \{(x, y) \sim D_{(X, Y)}\}$ is sampled from $D_{(X, Y)}$, denoted as $S \sim D_{(X, Y)}$. 

For a series of distributions $D_{(X_1, Y_1)}, \dots, D_{(X_T, Y_T)}$, we denote the mixture of the first $k$ distributions as $D_{[1:k]} = \frac{\sum_{i=1}^{k} \pi_i D_{(X_i, Y_i)}}{\sum_{i=1}^{k} \pi_i}$, where the mixture weights $\pi_1, \dots, \pi_T > 0$ with $\sum_k \pi_k = 1$.
For brevity, $D_k = D_{[k:k]} = D_{(X_k, Y_k)}$.

Denote $S|_{[k1:k2]} \overset{def}{=} \{s \in S | s \in supp\, D_{[k1:k2]}\}$. 
For simplicity, $S|_{k} = S|_{[k:k]}$.

Since continual learning tasks come one by one sequentially, we denote the hypothesis function that is found by an algorithm $\mathbf{A}$ after training the $k$-th task as $\mathbf{A}_k(S)$ with $S \sim D_{[1:k]}$. 
Strictly speaking, $h_k(x) = \mathbf{A}_k(S)(x)$ is only well-defined for $(x, y) \sim D_{[1:k]}$, and is not well-defined for $(x', y') \sim D_{k'},\, k' > k$. 
Even if some implementation may predict a real value by $h_k(x')$, we regard it as non-sense at time $k$ and only make sense until time $k'$. 

For the risk function, we will meet $\mathbf{R}_{D_{[k_1:k_2]}}(h_k)$ and we guarantee that $k_1 \leq k_2 \leq k$. 
Denote $$h_{k} = {\arg\max}_{1\leq i \leq k,j \in \{1, \dots\}} \{\dots, z_k^{i,j},\dots\},$$ where $z_k^{i,j} $ is the score function of the $j$-th class of the $i$-th task. 
The score function is any function that indicates which class the sample belongs to. For example, the score function could be the predicted logit of each class for a classification algorithm.
We calculate
\begin{align*}
\begin{split}
    &\mathbf{R}_{D_{[k_1:k_2]}}(h_{k}) = \mathbb{E}_{(x, y) \sim D_{[k_1:k_2]}} \\
    &[l({\arg\max}_{k_1\leq i \leq k_2,j \in \{1,\dots\}} \{\dots, z_k^{i,j}(x),\dots\}, y)].
\end{split}
\end{align*}
When we write $\mathbf{R}_{D^\alpha} (h)$ with $D^\alpha = (1 - \alpha) D + \alpha O$ (where $O$ denotes OOD and $\alpha \in [0, 1)$), we require $h$ to predict one additional OOD class as $$h_{k} = {\arg\max} \{\dots, z_k^{i,j},\dots; z_k^o \},$$ 
where $z_k^o$ is the score function of the OOD class. 

\begin{definition}[Fully-Observable Separated-Task Closed-World Learnability]
\label{def:full_obs_closed_learn}
    Given a set of distributions $\mathcal{D}$, a hypothesis function space $\mathcal{H}$, we say CIL is learnable 
    if there exists an algorithm $\mathbf{A}$ and a sequence $\{\epsilon_n | \lim_{n \rightarrow +\infty} \epsilon_n = 0\}$
    s.t. \textbf{(i)} for any $D_{1}, \dots, D_{T} \in \mathcal{D}$ with $supp\, D_{k} \cap supp\, D_{k'} = \emptyset, k \neq k'$, 
    and \textbf{(ii)} for any $\pi_1, \dots, \pi_T > 0$ with $\sum_k \pi_k = 1$,
    \begin{align*}
    \small
    \begin{split}
        \max_{k=1, \dots, T} \mathbb{E}_{S \sim D_{[1:k]}} [\mathbf{R}_{D_{[1:k]}} (\mathbf{A}_k(S)) - \inf_{h \in \mathcal{H}} \mathbf{R}_{D_{[1:k]}} (h)] < \epsilon_n.
    \end{split}
    \end{align*}
\end{definition}

We use $\epsilon$ to represent the error rate, where the index $n$ of $\epsilon_n$ represents the total number of samples. 
The equation $\lim_{n \rightarrow +\infty} \epsilon_n = 0$ means that the error rate decreases to $0$ as $n$ goes to $+\infty$.
\textbf{Definition} \ref{def:full_obs_closed_learn}, the risk function is calculated over $D_{[1:k]}$ at task $k$, which means the data of all the past tasks and the current task are observable for optimization.
It is a desirable property for CIL to take expectation over $D_{[1:k]}$ as it constructs a model that is equivalent to the model built with the full training data of all tasks. {Generally, when an algorithm satisfies \textbf{Definition} \ref{def:full_obs_closed_learn}, the system is already learnable because this is just the traditional supervised learning which can see/observe all the training data of all tasks and there is no OOD data involved (which means the \textit{closed-world}).
However, when we apply the algorithm $\mathbf{A}$ to solve for $\mathbf{A} (S)$ in practice, we usually cannot access all samples in $S$, which is partially-observable instead of fully-observable. {That is the case for continual learning as it assumes that in learning the new/current task, the training data of the previous/past tasks is not accessible, at least a major part of it.}

Due to the lack of full observations, we have to define $\mathbf{A}^r_k$ recursively. 
For any $S \sim D_{[1:k]}$, we define 
\begin{equation}
\label{eq:A_recur}
    \mathbf{A}^r_k (S) = \mathbf{A}^r_k (S|_k, \mathbf{A}^r_{k-1}(S|_{[1:k-1]})).
\end{equation}
The algorithm depends on implementation. In the following discussion, we assume that learning a new task does not interrupt the error bound of previous tasks. This is a valid assumption as existing algorithms \cite{Serra2018overcoming, supsup2020} achieve little or no forgetting. The version of \textbf{Definition} \ref{def:full_obs_closed_learn} for partial observations is as follows. 

\begin{definition}[Partially-Observable Separated-Task Closed-World Learnability]
\label{def:part_obs_closed_learn}
    Given a set of distributions $\mathcal{D}$, a hypothesis function space $\mathcal{H}$, we say CIL is learnable 
    if there exists an algorithm $\mathbf{A}$ and a sequence $\{\epsilon_n | \lim_{n \rightarrow +\infty} \epsilon_n = 0\}$
    s.t. \textbf{(i)} for any $D_{1}, \dots, D_{T} \in \mathcal{D}$ with $supp\, D_{k} \cap supp\, D_{k'} = \emptyset, k \neq k'$, 
    \textbf{(ii)} for any $\pi_1, \dots, \pi_T > 0$ with $\sum_k \pi_k = 1$,
    $$\max_{k=1, \dots, T} \mathbb{E}_{S \sim D_{[1:k]}} [\mathbf{R}_{D_{k}} (\mathbf{A}_k^r(S)) - \inf_{h \in \mathcal{H}} \mathbf{R}_{D_{k}} (h)] < \epsilon_n.$$
\end{definition}

In \textbf{Definition} \ref{def:full_obs_open_learn}, the risk function is calculated over $D_k$ alone as only the current task data $D_k$ is observable while the past tasks are not. It is desirable that \textbf{Definition} \ref{def:part_obs_closed_learn} implies \textbf{Definition} \ref{def:full_obs_closed_learn}, which transforms the learnability of a CIL problem into the learnability of a supervised problem.
Unfortunately, \textbf{Definition} \ref{def:part_obs_closed_learn} does not imply \textbf{Definition} \ref{def:full_obs_closed_learn}. 
\textbf{Theorem} \ref{thm:def2_not_imply_def1} shows that there exists a trivial hypothesis function that satisfies \textbf{Definition} \ref{def:part_obs_closed_learn} but doesn't satisfy \textbf{Definition} \ref{def:full_obs_closed_learn}.

\begin{theorem}[\textbf{Definition} \ref{def:part_obs_closed_learn} does not imply \textbf{Definition} \ref{def:full_obs_closed_learn}]
\label{thm:def2_not_imply_def1}
For a set of distributions $\mathcal{D}$ and a hypothesis function space $\mathcal{H}$, if \textbf{Definition} \ref{def:part_obs_closed_learn} holds and $\mathcal{H}$ has the capacity to learn more than one task, then there exists $h \in \mathcal{H}$ s.t. \textbf{Definition} \ref{def:part_obs_closed_learn} holds but \textbf{Definition} \ref{def:full_obs_closed_learn} doesn't hold. 
\end{theorem}

The proof is given in Appendix~\ref{sec.proof}. The main reason here is that only the samples of the current task are observable, while samples of both past and future tasks are hard to be exploited. 
From the perspective of forward looking, when training the current task, we have no access to any information of future tasks, where samples of future tasks are regarded as \textit{out-of-distribution} (OOD) samples with respect to the current and past tasks.
Inspired by \textbf{Theorem} \ref{thm:def2_not_imply_def1}, we include OOD detection into consideration and generalize \textbf{Definition} \ref{def:full_obs_closed_learn} to the open-world setting.

\begin{definition}[Fully-Observable Separated-Task Open-World Learnability]
\label{def:full_obs_open_learn}
    Given a set of distributions $\mathcal{D}$, a hypothesis function space $\mathcal{H}$, we say CIL is learnable 
    if there exists an algorithm $\mathbf{A}$ and a sequence $\{\epsilon_n | \lim_{n \rightarrow +\infty} \epsilon_n = 0\}$
    s.t. \textbf{(i)} for any $D_{1}, \dots, D_{T} \in \mathcal{D}$ with $supp\, D_{k} \cap supp\, D_{k'} = \emptyset, k \neq k'$, 
    \textbf{(ii)} for any $\pi_1, \dots, \pi_T > 0$ with $\sum_k \pi_k = 1$,
    and \textbf{(iii)} for any $O_{(X_1, Y_1)}, \dots, O_{(X_T, Y_T)} \in \mathcal{D}$, any $\alpha_1, \dots, \alpha_T \in [0, 1)$, 
    \begin{align*}
    \begin{split}
        \max_{k=1, \dots, T} \mathbb{E}_{S \sim D_{[1:k]}} [&\mathbf{R}_{D_{[1:k]}^{\alpha_{[1: k]}}} (\mathbf{A}_k(S)) \\
        &- \inf_{h \in \mathcal{H}} \mathbf{R}_{D_{[1:k]}^{\alpha_{[1: k]}}} (h)] < \epsilon_n, 
    \end{split}
    \end{align*}
    where $D_{[1:k]}^{\alpha_{[1: k]}} = \sum_{i=1}^k (1 - \alpha_i) D_{i} + \alpha_i O_{(X_i, Y_i)}$.
\end{definition}

The proof of \textbf{Definition} \ref{def:full_obs_open_learn} is guaranteed by previous work \cite{fang2022out}, which studies the learnablity of OOD detection. It's obvious that when \textbf{Definition} \ref{def:full_obs_open_learn} is satisfied, \textbf{Definition} \ref{def:full_obs_closed_learn} is satisfied, which is shown in \textbf{Theorem} \ref{thm:def3_imply_def1}.

\begin{theorem}[\textbf{Definition} \ref{def:full_obs_open_learn} implies \textbf{Definition} \ref{def:full_obs_closed_learn}]
\label{thm:def3_imply_def1}
For a set of distributions $\mathcal{D}$ and a hypothesis function space $\mathcal{H}$, if \textbf{Definition} \ref{def:full_obs_open_learn} holds, then \textbf{Definition} \ref{def:full_obs_closed_learn} holds. 
\end{theorem}

The proof is given in Appendix~\ref{sec.proof}. {When we have no access to samples of past tasks in practice, we define $\textbf{A}_k$ recursively as in Eq.~\ref{eq:A_recur}. 
The partial observable version of \textbf{Definition} \ref{def:full_obs_open_learn} is stated below.}
In \textbf{Definition} \ref{def:part_obs_open_learn}, the risk function is over $D_{k}$ instead of $D_{[1:k]}$ because it's the partial observable case.

\begin{definition}[Partially-Observable Separated-Task Open-World Learnability]
\label{def:part_obs_open_learn}
    Given a set of distributions $\mathcal{D}$, a hypothesis function space $\mathcal{H}$, we say CIL is learnable 
    if there exists an algorithm $\mathbf{A}$ and a sequence $\{\epsilon_n | \lim_{n \rightarrow +\infty} \epsilon_n = 0\}$
    s.t. \textbf{(i)} for any $D_{1}, \dots, D_{T} \in \mathcal{D}$ with $supp\, D_{k} \cap supp\, D_{k'} = \emptyset, k \neq k'$, 
    \textbf{(ii)} for any $\pi_1, \dots, \pi_T > 0$ with $\sum_k \pi_k = 1$,
    and \textbf{(iii)} for any $O_{(X_1, Y_1)}, \dots, O_{(X_T, Y_T)} \in \mathcal{D}$, any $\alpha_1, \dots, \alpha_T \in [0, 1)$, 
    $$\max_{k=1, \dots, T} \mathbb{E}_{S \sim D_{[1:k]}} [\mathbf{R}_{D_{k}^{\alpha_k}} (\mathbf{A}_k^r(S)) - \inf_{h \in \mathcal{H}} \mathbf{R}_{D_{k}^{\alpha_k}} (h)] < \epsilon_n, $$
    where $D_{k}^{\alpha_k} = (1 - \alpha_k) D_{k} + \alpha_k O_{(X_k, Y_k)}$.
\end{definition}

Note that \citet{fang2022out} showed that OOD detection is learnable
under a compatibility condition
for a single OOD detection problem and \textbf{Definition}~\ref{def:part_obs_open_learn} is about learnability with respect to an ensemble of multiple OOD detection problems. It is obvious that once each OOD detection problem is learnable, the ensemble of them is also learnable. With this definition, we derive that CIL is learnable as OOD detection is learnable. 
Different from \textbf{Theorem} \ref{thm:def2_not_imply_def1} that partially-observable learnability does not imply fully-observable learnability for the closed-world setting, \textbf{Theorem} \ref{thm:def4_imply_def3} shows that the learnability of a CIL system can be converted to a series of OOD learnability problems for the open-world setting (meaning there are OOD data). 

\begin{theorem}[\textbf{Definition} \ref{def:part_obs_open_learn} implies \textbf{Definition} \ref{def:full_obs_open_learn}]
\label{thm:def4_imply_def3}
For a set of distributions $\mathcal{D}$ and a hypothesis function space $\mathcal{H}$, if \textbf{Definition} \ref{def:part_obs_open_learn} holds, $\mathcal{H}$ enjoys enough capacity i.e. $\inf_{h \in \mathcal{H}} \mathbf{R}_{D_{[1:k]}^{\alpha_{[1: k]}}} (h) = 0$, and the loss function on all tasks is bounded by summation of loss function on each task i.e., Eq.~\ref{eq:loss_assump} in Appendix, then \textbf{Definition} \ref{def:full_obs_open_learn} holds and the upper bound $\epsilon_n$ is multiplied by $\max_{k=1, \dots, T} \sum_{t=1}^k \frac{\pi_{[t:T]}}{\pi_{[1:k]}}$. 
\end{theorem}

The proof is given in Appendix~\ref{sec.proof}. \textbf{Theorem}~\ref{thm:def4_imply_def3} connects \textbf{Definitions}~\ref{def:part_obs_open_learn} to  \ref{def:full_obs_open_learn} and \textbf{Theorem}~\ref{thm:def3_imply_def1} connects \textbf{Definitions}~\ref{def:full_obs_open_learn} to \ref{def:full_obs_closed_learn}, which is the desirable property of CIL. When all tasks have the same weight $\pi_1 = \dots = \pi_T = 1/T$, the multiplier $\max_{k=1, \dots, T} \sum_{t=1}^k \frac{\pi_{[t:T]}}{\pi_{[1:k]}} = T$, which is positively correlated with the number of tasks.

Though \textbf{Theorem} \ref{thm:def4_imply_def3} gives an upper bound to induce \textbf{Definition} \ref{def:full_obs_open_learn} from \textbf{Definition} \ref{def:part_obs_open_learn}, 
the hypothesis function that satisfies \textbf{Definition} \ref{def:full_obs_open_learn} is recursively derived from the previous tasks (see the proof). 
We can also observe that when tasks have different weights, the multiplier $\max_{k=1, \dots, T} \sum_{t=1}^k \frac{\pi_{[t:T]}}{\pi_{[1:k]}}$ depends on the order of tasks.
It is undesirable that the hypothesis function depends on the order of tasks.
When we can acquire some replay data of past tasks and treat them as OOD data, we have the following corollary that gives an order-free hypothesis function. 

\begin{corollary}
\label{thm:def4_imply_def3_replay}
For a set of distributions $\mathcal{D}$ and a hypothesis function space $\mathcal{H}$, if \textbf{Definition} \ref{def:part_obs_open_learn} holds, $\mathcal{H}$ enjoys enough capacity i.e. $\inf_{h \in \mathcal{H}} \mathbf{R}_{D_{[1:k]}^{\alpha_{[1: k]}}} (h) = 0$, and the loss function on all tasks is bounded by summation of the loss functions on every task i.e. Eq.~\ref{eq:loss_assump_replay} in Appendix, if we treat data of past tasks as OOD data, then \textbf{Definition} \ref{def:full_obs_open_learn} holds and the upper bound $\epsilon_n$ is multiplied by $\max_{k=1, \dots, T} \frac{k \pi_{[1:T]}}{\pi_{[1:k]}} $. 
\end{corollary}

The proof is given in Appendix~\ref{sec.proof}.

\section{Proposed Method} 
\label{sec.proposed}
The learnability in \textbf{Definition}~\ref{def:part_obs_open_learn} is defined over the OOD function of each task. By \textbf{Definition}~\ref{def:ood} of OOD, an OOD function is capable of classification (i.e., WP) for IND instances and rejection for OOD instances (or TP as it can be defined using OOD and such TP {differs from} OOD by a constant factor~\cite{kim2022theoretical}).
As discussed early, we use the masks in HAT~\cite{Serra2018overcoming} to protect each OOD model to ensure there is no forgetting. Following exactly this theoretical framework, an algorithm can be designed, which works quite well (see ROW (-WP) in Tab.~\ref{tab:ablation}). However, it is possible to do better by introducing a WP head so that we can use the OOD head for estimating only TP rather than for handling both WP and TP.

The proposed method ROW is a replay-based method. At each task $k$, the system receives dataset $D_k$ and leverages the replay data saved from previous tasks in the replay memory buffer $\mathcal{M}$ as the OOD data of the task to train an OOD detection head and also to fine-tune the WP head. 
Specifically, the model of each task has two heads: one for OOD (for computing TP) and one for WP. That is, we optimize the set of parameters $(\Psi_k, \theta_k, \phi_k)$, where $\Psi_k$ is the parameter set of the feature extractor $f_k$, $\theta_k$ is the parameter set of OOD head $h_k$, and $\phi_k$ is the parameter set of the WP head (i.e., classifier) $g_k$. The two task specific heads $h_k$ and $g_k$ receive {feature $u$} from the shared feature extractor $f_k$ and produce WP and TP probabilities, respectively. The training consists of three steps: 1) training the feature extractor $f_k$ and the OOD head $h_k$ using both IND instances in $D_k$ and OOD instances in $\mathcal{M}$ (i.e., the replay data), 2) fine-tuning a WP head $g_k$ for the task using $D_k$ based on only the fixed feature extractor $f_k$, and 3) fine-tuning the OOD heads of all tasks that have been learned so far. Training steps 2 and 3 are fast as both are simply fine-tuning the single layer of the classifiers (details below). 
The outputs from the two heads are used to compute the final CIL prediction probability in Eq.~\ref{eq:cil_in_til_and_tp}. An overview of the training and prediction process is illustrated in Fig.~\ref{fig:diagram}.

\begin{figure}
    \centering

\subfigure[Step 1]{

\tikzset{every picture/.style={line width=0.75pt}} 

\begin{tikzpicture}[x=0.65pt,y=0.65pt,yscale=-1,xscale=1]

\draw  [fill={rgb, 255:red, 121; green, 209; blue, 18 }  ,fill opacity=0.33 ] (256,98) -- (288,98) -- (288,183.8) -- (256,183.8) -- cycle ;
\draw  [fill={rgb, 255:red, 245; green, 166; blue, 35 }  ,fill opacity=0.56 ] (299,132.67) -- (331,132.67) -- (331,145.67) -- (299,145.67) -- cycle ;
\draw  [fill={rgb, 255:red, 74; green, 144; blue, 226 }  ,fill opacity=1 ] (341,133.33) -- (345.5,133.33) -- (345.5,142.67) -- (341,142.67) -- cycle ;
\draw    (288.5,135.79) -- (294.08,135.79) ;
\draw [shift={(297.08,135.79)}, rotate = 180] [fill={rgb, 255:red, 0; green, 0; blue, 0 }  ][line width=0.08]  [draw opacity=0] (3.57,-1.72) -- (0,0) -- (3.57,1.72) -- cycle    ;
\draw    (247.75,133.04) -- (252.67,133.04) ;
\draw [shift={(255.67,133.04)}, rotate = 180] [fill={rgb, 255:red, 0; green, 0; blue, 0 }  ][line width=0.08]  [draw opacity=0] (3.57,-1.72) -- (0,0) -- (3.57,1.72) -- cycle    ;
\draw  [fill={rgb, 255:red, 208; green, 2; blue, 27 }  ,fill opacity=1 ] (341,146.33) -- (345.5,146.33) -- (345.5,142.67) -- (341,142.67) -- cycle ;
\draw    (247.75,147.29) -- (252.67,147.29) ;
\draw [shift={(256.17,147.29)}, rotate = 180] [fill={rgb, 255:red, 0; green, 0; blue, 0 }  ][line width=0.08]  [draw opacity=0] (3.57,-1.72) -- (0,0) -- (3.57,1.72) -- cycle    ;
\draw    (288.5,142.79) -- (294.08,142.79) ;
\draw [shift={(297.08,142.79)}, rotate = 180] [fill={rgb, 255:red, 0; green, 0; blue, 0 }  ][line width=0.08]  [draw opacity=0] (3.57,-1.72) -- (0,0) -- (3.57,1.72) -- cycle    ;
\draw    (331.67,136.54) -- (337.5,136.51) ;
\draw [shift={(340.5,136.5)}, rotate = 179.72] [fill={rgb, 255:red, 0; green, 0; blue, 0 }  ][line width=0.08]  [draw opacity=0] (3.57,-1.72) -- (0,0) -- (3.57,1.72) -- cycle    ;
\draw    (331.7,143.54) -- (337.17,143.54) ;
\draw [shift={(340.17,143.54)}, rotate = 180] [fill={rgb, 255:red, 0; green, 0; blue, 0 }  ][line width=0.08]  [draw opacity=0] (3.57,-1.72) -- (0,0) -- (3.57,1.72) -- cycle    ;

\draw (237.3,129.78) node [anchor=north west][inner sep=0.75pt]  [font=\tiny]  {$x$};
\draw (268.42,137.01) node [anchor=north west][inner sep=0.75pt]  [font=\tiny]  {$f$};
\draw (310.00,133.9) node [anchor=north west][inner sep=0.75pt]  [font=\tiny]  {$h_{k}$};
\draw (237.32,146.28) node [anchor=north west][inner sep=0.75pt]  [font=\tiny]  {$x_{o}$};

\end{tikzpicture}
\label{fig:diagram_a}
}
\qquad
\subfigure[Step 2]{

\tikzset{every picture/.style={line width=0.75pt}} 

\begin{tikzpicture}[x=0.65pt,y=0.65pt,yscale=-1,xscale=1]

\draw    (422.75,133.79) -- (427.67,133.79) ;
\draw [shift={(430.67,133.79)}, rotate = 180] [fill={rgb, 255:red, 0; green, 0; blue, 0 }  ][line width=0.08]  [draw opacity=0] (3.57,-1.72) -- (0,0) -- (3.57,1.72) -- cycle    ;
\draw  [fill={rgb, 255:red, 121; green, 209; blue, 18 }  ,fill opacity=0.33 ] (430.33,92.28) -- (462.33,92.28) -- (462.33,178.08) -- (430.33,178.08) -- cycle ;
\draw  [fill={rgb, 255:red, 245; green, 166; blue, 35 }  ,fill opacity=0.56 ] (472.56,130.9) -- (504.56,130.9) -- (504.56,140.32) -- (472.56,140.32) -- cycle ;
\draw  [fill={rgb, 255:red, 74; green, 144; blue, 226 }  ,fill opacity=1 ] (524.2,130.7) -- (528.7,130.7) -- (528.7,140.04) -- (524.2,140.04) -- cycle ;
\draw    (512.09,135.87) -- (515.84,135.87) ;
\draw [shift={(518.84,135.87)}, rotate = 180] [fill={rgb, 255:red, 0; green, 0; blue, 0 }  ][line width=0.08]  [draw opacity=0] (3.57,-1.72) -- (0,0) -- (3.57,1.72) -- cycle    ;
\draw    (462.71,135.79) -- (468.3,135.79) ;
\draw [shift={(471.3,135.79)}, rotate = 180] [fill={rgb, 255:red, 0; green, 0; blue, 0 }  ][line width=0.08]  [draw opacity=0] (3.57,-1.72) -- (0,0) -- (3.57,1.72) -- cycle    ;

\draw (414.07,129.03) node [anchor=north west][inner sep=0.75pt]  [font=\tiny]  {$x$};
\draw (443.75,131.92) node [anchor=north west][inner sep=0.75pt]  [font=\tiny]  {$f$};
\draw (484.27,131.05) node [anchor=north west][inner sep=0.75pt]  [font=\tiny]  {$g_{k}$};

\end{tikzpicture}
\label{fig:diagram_b}
}
\qquad
\subfigure[Step 3]{

\tikzset{every picture/.style={line width=0.75pt}} 

\begin{tikzpicture}[x=0.65pt,y=0.65pt,yscale=-1,xscale=1]

\draw  [fill={rgb, 255:red, 121; green, 209; blue, 18 }  ,fill opacity=0.33 ] (256,193.87) -- (288,193.87) -- (288,279.67) -- (256,279.67) -- cycle ;
\draw  [fill={rgb, 255:red, 74; green, 144; blue, 226 }  ,fill opacity=1 ] (341,209.75) -- (345.5,209.75) -- (345.5,219.08) -- (341,219.08) -- cycle ;
\draw  [fill={rgb, 255:red, 208; green, 2; blue, 27 }  ,fill opacity=1 ] (341,222.75) -- (345.5,222.75) -- (345.5,219.08) -- (341,219.08) -- cycle ;
\draw  [fill={rgb, 255:red, 74; green, 144; blue, 226 }  ,fill opacity=1 ] (341,254.83) -- (345.5,254.83) -- (345.5,264.17) -- (341,264.17) -- cycle ;
\draw  [fill={rgb, 255:red, 208; green, 2; blue, 27 }  ,fill opacity=1 ] (341,267.83) -- (345.5,267.83) -- (345.5,264.17) -- (341,264.17) -- cycle ;
\draw    (247.75,228.04) -- (252.67,228.04) ;
\draw [shift={(254,228.04)}, rotate = 180] [fill={rgb, 255:red, 0; green, 0; blue, 0 }  ][line width=0.08]  [draw opacity=0] (3.57,-1.72) -- (0,0) -- (3.57,1.72) -- cycle    ;
\draw    (247.75,242.29) -- (252.67,242.29) ;
\draw [shift={(254.5,242.29)}, rotate = 180] [fill={rgb, 255:red, 0; green, 0; blue, 0 }  ][line width=0.08]  [draw opacity=0] (3.57,-1.72) -- (0,0) -- (3.57,1.72) -- cycle    ;
\draw    (288.5,213.46) -- (294.08,213.46) ;
\draw [shift={(297.08,213.46)}, rotate = 180] [fill={rgb, 255:red, 0; green, 0; blue, 0 }  ][line width=0.08]  [draw opacity=0] (3.57,-1.72) -- (0,0) -- (3.57,1.72) -- cycle    ;
\draw    (288.5,220.46) -- (294.08,220.46) ;
\draw [shift={(297.08,220.46)}, rotate = 180] [fill={rgb, 255:red, 0; green, 0; blue, 0 }  ][line width=0.08]  [draw opacity=0] (3.57,-1.72) -- (0,0) -- (3.57,1.72) -- cycle    ;
\draw    (288.5,258.13) -- (294.08,258.13) ;
\draw [shift={(297.08,258.13)}, rotate = 180] [fill={rgb, 255:red, 0; green, 0; blue, 0 }  ][line width=0.08]  [draw opacity=0] (3.57,-1.72) -- (0,0) -- (3.57,1.72) -- cycle    ;
\draw    (288.5,265.13) -- (294.08,265.13) ;
\draw [shift={(297.08,265.13)}, rotate = 180] [fill={rgb, 255:red, 0; green, 0; blue, 0 }  ][line width=0.08]  [draw opacity=0] (3.57,-1.72) -- (0,0) -- (3.57,1.72) -- cycle    ;
\draw  [fill={rgb, 255:red, 245; green, 166; blue, 35 }  ,fill opacity=0.56 ] (299,210) -- (331,210) -- (331,223) -- (299,223) -- cycle ;
\draw  [fill={rgb, 255:red, 245; green, 166; blue, 35 }  ,fill opacity=0.56 ] (299,254) -- (331,254) -- (331,267) -- (299,267) -- cycle ;
\draw    (331.7,213.21) -- (337.53,213.18) ;
\draw [shift={(340.53,213.17)}, rotate = 179.72] [fill={rgb, 255:red, 0; green, 0; blue, 0 }  ][line width=0.08]  [draw opacity=0] (3.57,-1.72) -- (0,0) -- (3.57,1.72) -- cycle    ;
\draw    (331.73,220.21) -- (337.2,220.21) ;
\draw [shift={(340.2,220.21)}, rotate = 180] [fill={rgb, 255:red, 0; green, 0; blue, 0 }  ][line width=0.08]  [draw opacity=0] (3.57,-1.72) -- (0,0) -- (3.57,1.72) -- cycle    ;
\draw    (331.67,257.88) -- (337.5,257.85) ;
\draw [shift={(340.5,257.83)}, rotate = 179.72] [fill={rgb, 255:red, 0; green, 0; blue, 0 }  ][line width=0.08]  [draw opacity=0] (3.57,-1.72) -- (0,0) -- (3.57,1.72) -- cycle    ;
\draw    (331.7,264.88) -- (337.17,264.88) ;
\draw [shift={(340.17,264.88)}, rotate = 180] [fill={rgb, 255:red, 0; green, 0; blue, 0 }  ][line width=0.08]  [draw opacity=0] (3.57,-1.72) -- (0,0) -- (3.57,1.72) -- cycle    ;

\draw (269.08,233.51) node [anchor=north west][inner sep=0.75pt]  [font=\tiny]  {$f$};
\draw (320.5,227) node [anchor=north west][inner sep=0.75pt]  [rotate=-90]  {$\cdots $};
\draw (237.3,224.78) node [anchor=north west][inner sep=0.75pt]  [font=\tiny]  {$x$};
\draw (237.3,241.28) node [anchor=north west][inner sep=0.75pt]  [font=\tiny]  {$x_{o}$};
\draw (310.00,211.23) node [anchor=north west][inner sep=0.75pt]  [font=\tiny]  {$h_{k}$};
\draw (310.00,255.23) node [anchor=north west][inner sep=0.75pt]  [font=\tiny]  {$h_{k}$};

\end{tikzpicture}
\label{fig:diagram_c}
}
\qquad
\subfigure[Inference]{

\tikzset{every picture/.style={line width=0.75pt}} 

\begin{tikzpicture}[x=0.65pt,y=0.65pt,yscale=-1,xscale=1]

\draw    (462.75,193.67) -- (468.08,185.74) ;
\draw [shift={(469.75,183.25)}, rotate = 123.9] [fill={rgb, 255:red, 0; green, 0; blue, 0 }  ][line width=0.08]  [draw opacity=0] (3.57,-1.72) -- (0,0) -- (3.57,1.72) -- cycle    ;
\draw    (422.25,208.8) -- (427.17,208.8) ;
\draw [shift={(430.17,208.8)}, rotate = 180] [fill={rgb, 255:red, 0; green, 0; blue, 0 }  ][line width=0.08]  [draw opacity=0] (3.57,-1.72) -- (0,0) -- (3.57,1.72) -- cycle    ;
\draw    (463.25,224) -- (469.53,229.73) ;
\draw [shift={(471.75,231.75)}, rotate = 222.36] [fill={rgb, 255:red, 0; green, 0; blue, 0 }  ][line width=0.08]  [draw opacity=0] (3.57,-1.72) -- (0,0) -- (3.57,1.72) -- cycle    ;
\draw  [fill={rgb, 255:red, 121; green, 209; blue, 18 }  ,fill opacity=0.33 ] (430.3,167.54) -- (462.3,167.54) -- (462.3,253.34) -- (430.3,253.34) -- cycle ;
\draw  [fill={rgb, 255:red, 245; green, 166; blue, 35 }  ,fill opacity=0.56 ] (472.6,170.34) -- (504.6,170.34) -- (504.6,179.75) -- (472.6,179.75) -- cycle ;
\draw  [fill={rgb, 255:red, 245; green, 166; blue, 35 }  ,fill opacity=0.56 ] (472.6,183.34) -- (504.6,183.34) -- (504.6,196.34) -- (472.6,196.34) -- cycle ;
\draw  [fill={rgb, 255:red, 74; green, 144; blue, 226 }  ,fill opacity=1 ] (513.15,183.42) -- (517.65,183.42) -- (517.65,192.75) -- (513.15,192.75) -- cycle ;
\draw  [fill={rgb, 255:red, 208; green, 2; blue, 27 }  ,fill opacity=1 ] (513.15,196.42) -- (517.65,196.42) -- (517.65,192.75) -- (513.15,192.75) -- cycle ;
\draw  [fill={rgb, 255:red, 74; green, 144; blue, 226 }  ,fill opacity=1 ] (513.15,170.42) -- (517.65,170.42) -- (517.65,179.75) -- (513.15,179.75) -- cycle ;
\draw  [fill={rgb, 255:red, 74; green, 144; blue, 226 }  ,fill opacity=1 ] (524.15,178.17) -- (528.65,178.17) -- (528.65,187.5) -- (524.15,187.5) -- cycle ;
\draw    (505.48,175.5) -- (509.23,175.5) ;
\draw [shift={(512.2,175.5)}, rotate = 180] [fill={rgb, 255:red, 0; green, 0; blue, 0 }  ][line width=0.08]  [draw opacity=0] (3.57,-1.72) -- (0,0) -- (3.57,1.72) -- cycle    ;
\draw    (505.45,189.75) -- (509.2,189.75) ;
\draw [shift={(512.2,189.75)}, rotate = 180] [fill={rgb, 255:red, 0; green, 0; blue, 0 }  ][line width=0.08]  [draw opacity=0] (3.57,-1.72) -- (0,0) -- (3.57,1.72) -- cycle    ;
\draw  [fill={rgb, 255:red, 245; green, 166; blue, 35 }  ,fill opacity=0.56 ] (472.6,222.09) -- (504.6,222.09) -- (504.6,231.5) -- (472.6,231.5) -- cycle ;
\draw  [fill={rgb, 255:red, 245; green, 166; blue, 35 }  ,fill opacity=0.56 ] (472.6,235.09) -- (504.6,235.09) -- (504.6,248.09) -- (472.6,248.09) -- cycle ;
\draw  [fill={rgb, 255:red, 74; green, 144; blue, 226 }  ,fill opacity=1 ] (513.18,235.17) -- (517.68,235.17) -- (517.68,244.5) -- (513.18,244.5) -- cycle ;
\draw  [fill={rgb, 255:red, 208; green, 2; blue, 27 }  ,fill opacity=1 ] (513.18,248.17) -- (517.68,248.17) -- (517.68,244.5) -- (513.18,244.5) -- cycle ;
\draw  [fill={rgb, 255:red, 74; green, 144; blue, 226 }  ,fill opacity=1 ] (513.18,222.17) -- (517.68,222.17) -- (517.68,231.5) -- (513.18,231.5) -- cycle ;
\draw  [fill={rgb, 255:red, 74; green, 144; blue, 226 }  ,fill opacity=1 ] (524.18,229.92) -- (528.68,229.92) -- (528.68,239.25) -- (524.18,239.25) -- cycle ;
\draw    (505.43,227) -- (509.18,227) ;
\draw [shift={(512.18,227)}, rotate = 180] [fill={rgb, 255:red, 0; green, 0; blue, 0 }  ][line width=0.08]  [draw opacity=0] (3.57,-1.72) -- (0,0) -- (3.57,1.72) -- cycle    ;
\draw    (505.48,241.5) -- (509.23,241.5) ;
\draw [shift={(512.23,241.5)}, rotate = 180] [fill={rgb, 255:red, 0; green, 0; blue, 0 }  ][line width=0.08]  [draw opacity=0] (3.57,-1.72) -- (0,0) -- (3.57,1.72) -- cycle    ;

\draw (414.1,205.45) node [anchor=north west][inner sep=0.75pt]  [font=\tiny]  {$x$};
\draw (443.25,205.18) node [anchor=north west][inner sep=0.75pt]  [font=\tiny]  {$f$};
\draw (483.53,184.65) node [anchor=north west][inner sep=0.75pt]  [font=\tiny]  {$h_{1}$};
\draw (484.32,170.48) node [anchor=north west][inner sep=0.75pt]  [font=\tiny]  {$g_{1}$};
\draw (496.27,198) node [anchor=north west][inner sep=0.75pt]  [rotate=-90]  {$\cdots $};
\draw (483.53,236.4) node [anchor=north west][inner sep=0.75pt]  [font=\tiny]  {$h_{k}$};
\draw (484.32,222.23) node [anchor=north west][inner sep=0.75pt]  [font=\tiny]  {$g_{k}$};

\end{tikzpicture}
\label{fig:diagram_d}
}
    \caption{
    Overview of training steps at task $k$ and inference. (a): the first step of training the feature extractor and OOD head for task $k$. The system receives both IND instance $x \in D_k$ and OOD instance $x_o \in \mathcal{M}$. The output has IND classes (in blue) and the OOD class or label (in red). (b): the second step of fine-tuning the WP head using the IND training data only. (c): fine-tuning all OOD heads using both IND and OOD instances. (d): inference/prediction. For a test instance $x$, obtain TP and WP probabilities, and compute the CIL probability as in Eq.~\ref{eq:cil_in_til_and_tp}. }
\label{fig:diagram}
\end{figure}
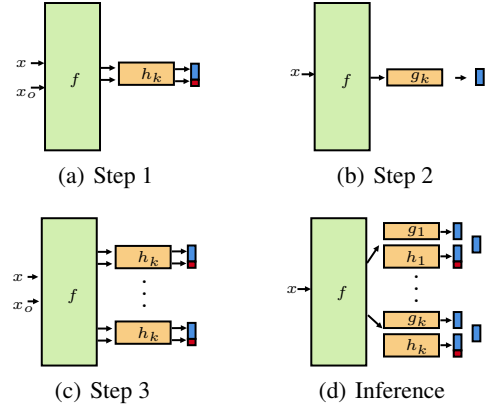

\textbf{1) Training Feature Extractor and OOD Head.} This step trains the OOD head $h_k$ for task $k$. Its feature extractor $f_k$ is also shared by the WP head (see below). An illustration of the training process is given in Fig.~\ref{fig:diagram_a}. Since OOD instances are any instances whose labels do not belong to task $k$, we leverage the task data $D_k$ as IND instances and the saved replay instances of tasks $k' \neq k$ in the memory buffer $\mathcal{M}$ as OOD instances represented by an OOD class (in red) in the OOD head. The network $h_k \circ f_k$ is trained to maximize the probability $p(y | x, k)=\text{softmax} h_k (f (x, k; \Psi_k); \theta_k)_y$ for an IND instance $x \in D_k$ and maximize the probability $p(ood | x, k)$ for OOD instance $x \in \mathcal{M}$. Formally, this is achieved by minimizing the sum of cross-entropy losses
\begin{align}
\small
\begin{split}
    \mathcal{L}_{ood}(\Psi_t, \theta_k) = &- \frac{1}{2N} \Big( \sum_{(x, y) \in D_k} \log p(y | x, k) \\ 
    &+ \sum_{(x, y) \in \mathcal{M}} \log p(ood | x, k) \Big) , \label{eq.feature_train}
\end{split}
\end{align}
where $N$ is the number of instances in $D_k$. 
We utilize upsampling with the replay instances to achieve an equal number of samples as the current task data $\mathcal{D}_k$.
The first loss is to discriminate the IND instances while the second loss is introduced to distinguish between IND and OOD instances.  

To deal with forgetting, we use the HAT method~\cite{Serra2018overcoming} (see Appendix~\ref{sec:hat}).

\textbf{2) Fine-Tuning the WP Head.}
Given the feature extractor trained in the first step, we fix the feature extractor and fine-tune the WP head $g_k$ (i.e., the WP classifier) using only $D_k$ by adjusting the parameters $\phi_k$. 
This is achieved by minimizing the cross-entropy loss
\begin{align}
	\mathcal{L}_{WP}(\phi_k) = - \frac{1}{N} \sum_{(x, y) \in D_k} \log p(y | x, k).
\end{align}
\textbf{WP probabilities} for the classes of task $k$ are just the output softmax probabilities.

\textbf{3) Fine-Tuning the OOD Heads of All Tasks.}
The OOD head $h_k$ built in step 1) is biased because for early tasks, where the instances in $\mathcal{M}$ are less diverse, the OOD heads for them will be weaker than the OOD heads of later tasks when the instances in $\mathcal{M}$ are more diverse. To mitigate this bias, we fine-tune all OOD heads of all tasks after training each task using only the replay data in $\mathcal{M}$. After training task $k$, we have $\mathcal{M}$ with replay instances of classes from task $1$ to $k$. For each task $k' \leq k$, reconstruct a new IND data $\tilde{D}_{k'}$ by selecting instances corresponding to task $k'$ from ${\mathcal{M}}$, and a new pseudo memory buffer $\tilde{\mathcal{M}}$ after removing the instances in $\tilde{D}_{k'}$. 
We then fine-tune every OOD head by minimizing the loss function
\begin{align}
\small
\begin{split}
	\mathcal{L}_{TP}(\theta_{k'}) - \frac{1}{M} \Big( \sum_{(x, y) \in \tilde{\mathcal{M}}} \log p(ood | x, k') \\
     + \sum_{(x, y) \in \tilde{D}_{k'}} \log p(y | x, k') \Big) \label{eq.obj_back_update}
\end{split}
\end{align}
where $M$ is $|\tilde{D}_{k}| + |\tilde{\mathcal{M}}|$.
Although the TP probability can be defined simply using the fine-tuned OOD heads, it can be further improved, which we discuss next.

\subsection{Distance-Based Coefficient} \label{sec.dist_coef}
We can further improve the performance by incorporating a distance-based coefficient defined at the feature level into the output from the OOD head. The intuition is based on the observation that samples identified as OOD using a score function defined at the feature level are not recognized with a score function defined in the output level, and vice versa~\cite{wang2022vim}. Their combination usually produces a better OOD detector.

After training task $k$, compute the means of the feature vectors per class of the task and the variance of the features. Denote the mean of class $y$ by $\mu_y$ and the variance by $\Sigma_k = \sum_y \Sigma_y$, where $\Sigma_y$ is the variance of features of class $y$. Using Mahalanobis distance (MD), the coefficient of an instance $x$ for task $k$ is 
\begin{align}
    c_k (x) = \max_y 1/MD(x ; \mu_y, \Sigma_k).
\end{align}
The coefficient is large if the feature of a test instance $x$ is close to one of the sample means of the task and small otherwise.

We finally define the \textbf{TP probability} for task $k$ as 
\begin{align}
    \mathbf{P}(X_k|x) = c_k (x) \max_{j} \text{softmax} (h_k(x))_j / Z,
\end{align}
where $Z$ is the normalizing factor and the maximum is taken over the softmax outputs of the IND classes $j$ obtained by the OOD head $h_k$. The $\max_j$ operation can also be seen as the maximum softmax probability score~\cite{hendrycks2016baseline}.

With the WP and TP probabilities, we now make a CIL prediction based on Eq.~\ref{eq:cil_in_til_and_tp}.

\section{Empirical Evaluation}
\textbf{Baselines.} We compare the proposed ROW\footnote{\url{https://github.com/k-gyuhak/CLOOD}} with 12 baselines. Five are exemplar-free (i.e., saving no previous data) methods and seven are replay-based methods. The exemplar-free methods are: \textbf{HAT}~\cite{Serra2018overcoming}, \textbf{OWM}~\cite{zeng2019continuous}, \textbf{SLDA}~\cite{hayes2020lifelong}, \textbf{PASS}~\cite{Zhu_2021_CVPR_pass}, and \textbf{L2P}~\cite{wang2022learning_l2p}. For the multi-head method HAT, we make prediction by taking $\arg\max$ over the concatenated logits from each task model as \cite{kim2022theoretical}. The replay methods are: \textbf{iCaRL}~\cite{Rebuffi2017}, \textbf{A-GEM}~\cite{Chaudhry2019ICLR}, \textbf{EEIL}~\cite{castro2018end},
\textbf{GD}~\citep{lee2019overcoming} without external data, 
\textbf{DER++}~\cite{NEURIPS2020_b704ea2c_derpp}, \textbf{HAL}~\citep{Chaudhry_Gordo_Dokania_Torr_Lopez-Paz_2021_hal}, and \textbf{MORE}~\cite{kim2022multi}. 

We could not make the recent system in \citep{Wu_2022_CVPR} using a pre-trained model as no code is released. We also do not include the existing parameter isolation methods that deal with CIL problems as they are very weak. 
HyperNet~\citep{von2019continual_hypernet} and PR~\citep{henning2021posterior} find the task-id via an entropy function and SupSup~\citep{supsup2020} finds it via gradient update. They then make a within-task prediction.
SupSup, PR, and iTAML~\citep{rajasegaran2020itaml} assume the test instances come in batches and all samples in a batch belong to one task. When tested per sample, 
HyperNet, SupSup, PR and iTAML achieve 22.4, 11.8, 45.2 and 33.5 on 10 tasks of CIFAR100, respectively, which are much lower than 51.4 of iCaRL. CCG~\citep{abati2020conditional} has no code. The systems 
in~\cite{kim2022theoretical} are also not included because they are quite weak as their contrastive learning does not work well with a pre-trained model.
The results reported in their paper based on ResNet-18 are also weaker than ROW.

\textbf{Datasets.} We use three popular continual learning benchmark datasets.
\textbf{1). CIFAR10}~\citep{Krizhevsky2009learning}. This is an image classification dataset consisting of 60,000 color images of size 32x32, among which 50,000 are training data and 10,000 are testing data. It has 10 different classes. 
\textbf{2). CIFAR100}~\citep{Krizhevsky2009learning}. This dataset consists of 50,000 training images and 10,000 testing images with 100 classes. Each image is colored and of size 32x32. 
\textbf{3). Tiny-ImageNet}~\citep{Le2015TinyIV}. This dataset has 200 classes with 500 training images of size 64x64 per class. The validation data has 50 samples per class. Since no label is provided for the test data, we use the validation set for testing as in~\citep{Zhu_2021_CVPR_pass}.

\textbf{Backbone Architecture and Pre-Training.}
We use the backbone architecture of transformer DeiT-S/16~\citep{touvron2021training_deit}. As pre-training models or feature extractors are increasingly used in all types of applications, including continual learning~\cite{wang2022learning_l2p,kim2022multi,Wu_2022_CVPR}, we also take this approach. Following~\cite{kim2022multi}, to ensure there is no information leak from pre-training to continual learning, the pre-trained model/network is trained using 611 classes of ImageNet
after removing 389 classes which are similar or identical to the classes of CIFAR and Tiny-ImageNet.
To leverage the strong performance of the pre-trained model while adapting to new knowledge, we fix the feature extractor and append trainable \textbf{adapter modules} of fully-connected networks with one hidden layer at each transformer layer~\citep{Houlsby2019Parameter} except SLDA and L2P.\footnote{For SLDA and L2P, we follow the original papers. SLDA fine-tunes only the classifier with a fixed feature extractor and L2P trains learnable prompts.} The number of neurons in each hidden layer is 64 for CIFAR10 and 128 for other datasets.
Note that \textbf{\textit{all baselines and ROW use the same architecture and the same pre-training model for fairness}} as using a pre-trained model improves the performance~\citep{kim2022multi,ostapenko2022foundational}. 

Note that we do not use the pre-trained models like CLIP~\citep{radford2021learning} or others trained using the full ImageNet data due to \textbf{information leak} both in terms of features and class labels because our experiment data have been used in training these pre-trained models. This leakage can seriously affect the results. For example, the L2P system  using the pre-training model trained using the full ImageNet data performs extremely well, but after those overlapping classes are removed in pre-training, its performances drop greatly. In Tab.~\ref{Tab:maintable}, we can see that it is in fact quite weak.

\textbf{Training Details.} 
For the replay-based methods, we follow the budget sizes of \cite{Rebuffi2017,NEURIPS2020_b704ea2c_derpp}.
For our method, we use the memory budget strategy~\cite{chaudhry2019continual} to save equal number of samples per class.
Denote the budget size by $|\mathcal{M}|$. 

For CIFAR10, we split the 10 classes into 5 tasks with 2 classes per task. 
We refer the experiment as C10-5T. The memory budget size $|\mathcal{M}|$ is 200.

For CIFAR100, we conduct two experiments. We split the 100 classes into 10 and 20 tasks, where each task has 10 classes and 5 classes, respectively. We refer the experiments as C100-10T and C100-20T. We choose $|\mathcal{M}| =$ 2,000.

For Tiny-ImageNet, we conduct two experiments. We split the 200 classes into 5 tasks with 40 classes per task and 10 tasks with 20 classes per task. We refer the experiments as T-5T and T-10T, respectively. We save 2,000 samples in total for both experiments.

Following the random class order protocol of the existing methods~\citep{Rebuffi2017,lee2019overcoming}, we randomly generate 5 different class orders for each experiment and report the average accuracy over the 5 random orders.

For all the experiments of our system,  we find a good set of learning rates and the number of epochs via validation data made of 10\% of the training data. The hyper-parameters of our system is reported in Appendix~\ref{sec.hyper_param}. For the baselines, we use the experiment setups as reported in their official papers. If we could not reproduce the result, we find the hyper-parameters via the validation set.

\textbf{Evaluation Metrics.} We use two metrics: average classification accuracy (ACA) and average forgetting rate. ACA after the last task $t$ is $\mathcal{A}_{t} = \sum_{i=1}^{t} A_{i}^{t}/t$, where $A_{i}$ is the accuracy of the model on task $i$th data after learning task $t$. The average forgetting rate after task $t$ is $\mathcal{F}_{t} = \sum_{i=1}^{t-1} A_{i}^{i} - A_{i}^{t}$ \cite{liu2020mnemonics}. This is also referred as backward transfer in other literature~\citep{Lopez2017gradient}. 

\subsection{Results and Comparison} \label{sec.results_comparison}
\begin{table}
\vspace{-2mm}
\caption{
Average classification accuracy after the final task. `-XT' means X number of tasks. Our system ROW and all baselines use the pre-trained network. The last 7 baselines are replay-based systems. The last column shows the average of each row. We highlight the best results in each column in bold.
}
\label{Tab:maintable}
\centering
\resizebox{1.\columnwidth}{!}{
\begin{tabular}{l c c c c c c }
\toprule
\multirow{1}{*}{Method}  & \multicolumn{1}{c}{C10-5T}  &  \multicolumn{1}{c}{C100-10T} &  \multicolumn{1}{c}{C100-20T} &  \multicolumn{1}{c}{T-5T} & \multicolumn{1}{c}{T-10T} & \multicolumn{1}{c}{Average}\\
\midrule
HAT         & 79.36\scalebox{0.5}{$\pm$5.16} & 68.99\scalebox{0.5}{$\pm$0.21} & 61.83\scalebox{0.5}{$\pm$0.62} & \textbf{65.85}\scalebox{0.5}{$\pm$0.60} & 62.05\scalebox{0.5}{$\pm$0.55} & 67.62 \\
OWM         & 41.69\scalebox{0.5}{$\pm$6.34} & 21.39\scalebox{0.5}{$\pm$3.18} & 16.98\scalebox{0.5}{$\pm$4.44} & 24.55\scalebox{0.5}{$\pm$2.48} & 17.52\scalebox{0.5}{$\pm$3.45} & 24.43 \\
SLDA        & 88.64\scalebox{0.5}{$\pm$0.05} & 67.82\scalebox{0.5}{$\pm$0.05} & 67.80\scalebox{0.5}{$\pm$0.05} & 57.93\scalebox{0.5}{$\pm$0.05} & 57.93\scalebox{0.5}{$\pm$0.06} & 68.02 \\
PASS        & 86.21\scalebox{0.5}{$\pm$1.10} & 68.90\scalebox{0.5}{$\pm$0.94} & 66.77\scalebox{0.5}{$\pm$1.18} & 61.03\scalebox{0.5}{$\pm$0.38} & 58.34\scalebox{0.5}{$\pm$0.42} & 68.25 \\
L2P         & 73.59\scalebox{0.5}{$\pm$4.15} & 61.72\scalebox{0.5}{$\pm$0.81} & 53.84\scalebox{0.5}{$\pm$1.59} & 59.12\scalebox{0.5}{$\pm$0.96} & 54.09\scalebox{0.5}{$\pm$1.14} & 60.47 \\
\hline
iCaRL       & 87.55\scalebox{0.5}{$\pm$0.99} & 68.90\scalebox{0.5}{$\pm$0.47} & 69.15\scalebox{0.5}{$\pm$0.99} & 53.13\scalebox{0.5}{$\pm$1.04} & 51.88\scalebox{0.5}{$\pm$2.36} & 66.12 \\ 
A-GEM       & 56.33\scalebox{0.5}{$\pm$7.77} & 25.21\scalebox{0.5}{$\pm$4.00} & 21.99\scalebox{0.5}{$\pm$4.01} & 30.53\scalebox{0.5}{$\pm$3.99} & 21.90\scalebox{0.5}{$\pm$5.52} & 36.89 \\
EEIL        & 82.34\scalebox{0.5}{$\pm$3.13} & 68.08\scalebox{0.5}{$\pm$0.51} & 63.79\scalebox{0.5}{$\pm$0.66} & 53.34\scalebox{0.5}{$\pm$0.54} & 50.38\scalebox{0.5}{$\pm$0.97} & 63.59 \\
GD          & 89.16\scalebox{0.5}{$\pm$0.53} & 64.36\scalebox{0.5}{$\pm$0.57} & 60.10\scalebox{0.5}{$\pm$0.74} & 53.01\scalebox{0.5}{$\pm$0.97} & 42.48\scalebox{0.5}{$\pm$2.53} & 61.82 \\
DER++       & 84.63\scalebox{0.5}{$\pm$2.91} & 69.73\scalebox{0.5}{$\pm$0.99} & 70.03\scalebox{0.5}{$\pm$1.46} & 55.84\scalebox{0.5}{$\pm$2.21} & 54.20\scalebox{0.5}{$\pm$3.28} & 66.89 \\
HAL         & 84.38\scalebox{0.5}{$\pm$2.70} & 67.17\scalebox{0.5}{$\pm$1.50} & 67.37\scalebox{0.5}{$\pm$1.45} & 52.80\scalebox{0.5}{$\pm$2.37} & 55.25\scalebox{0.5}{$\pm$3.60} & 65.39 \\
MORE         & 89.16\scalebox{0.5}{$\pm$0.96} & 70.23\scalebox{0.5}{$\pm$2.27} & 70.53\scalebox{0.5}{$\pm$1.09} & 64.97\scalebox{0.5}{$\pm$1.28} & 63.06\scalebox{0.5}{$\pm$1.26} & 71.59 \\
\hline
ROW        & \textbf{90.97}\scalebox{0.5}{$\pm$0.19} & \textbf{74.72}\scalebox{0.5}{$\pm$0.48} & \textbf{74.60}\scalebox{0.5}{$\pm$0.12} & 65.11\scalebox{0.5}{$\pm$1.97} & \textbf{63.21}\scalebox{0.5}{$\pm$2.53} & \textbf{73.72} \Tstrut \\
\bottomrule
\end{tabular}
}
\vspace{-3mm}
\end{table}

\textbf{Average Classification Accuracy.} 
Tab.~\ref{Tab:maintable} shows the average classification accuracy after the final task. The last column Average indicates the average performance of each method over the 5 experiments. Our proposed method ROW performs the best consistently. On average, ROW achieves 73.72\% while the best replay baseline (MORE) achieves 71.59\%. We observe that MORE is significantly better than the other baselines. This is because MORE actually builds an OOD model for each task, which is close to the proposed theory but less principled than ROW.

The baselines SLDA and L2P are proposed to leverage a strong pre-trained feature extractor in the original papers. SLDA freezes the feature extractor and only fine-tunes the classifier. It performs well for the simple experiment C10-5T, but is significantly poorer than ROW on all experiments. 
This is because the fixed feature extractor does not adapt to new knowledge. 
Our method updates the feature extractor via adapter modules to new knowledge and it is able to learn more complex problems.
L2P trains a set of prompt embeddings. In the original paper, it uses a feature extractor that was pre-trained with ImageNet-21k which already includes the classes of the continual learning (CL) evaluation datasets (information leak). After we remove the classes similar to the datasets used in CL, its performance drops dramatically (60.47\% on average over the 5 experiments) and much poorer than our ROW (73.72\% on average).

\vspace{-1mm}
We conduct additional experiments with the size of memory buffer reduced by half to show the effectiveness of our method. The new memory buffer sizes for CIFAR10, CIFAR100, and Tiny-ImageNet are 100, 1,000, and 1,000, respectively. Tab.~\ref{Tab:smaller_memory} shows that our method ROW experiences little reduction in performance whereas the other replay-based baselines suffer from significant performance reduction. On average over the 5 experiments, ROW achieves 72.70\% while previously with the original memory buffer, it achieved 73.72. In contrast, the second best baseline DER++ reduces from 66.89 to 62.16. MORE is also robust with small memory sizes, but its average accuracy is 71.44 which is still lower than ROW.
\begin{figure}
    \centering
    \includegraphics[width=2.2in]{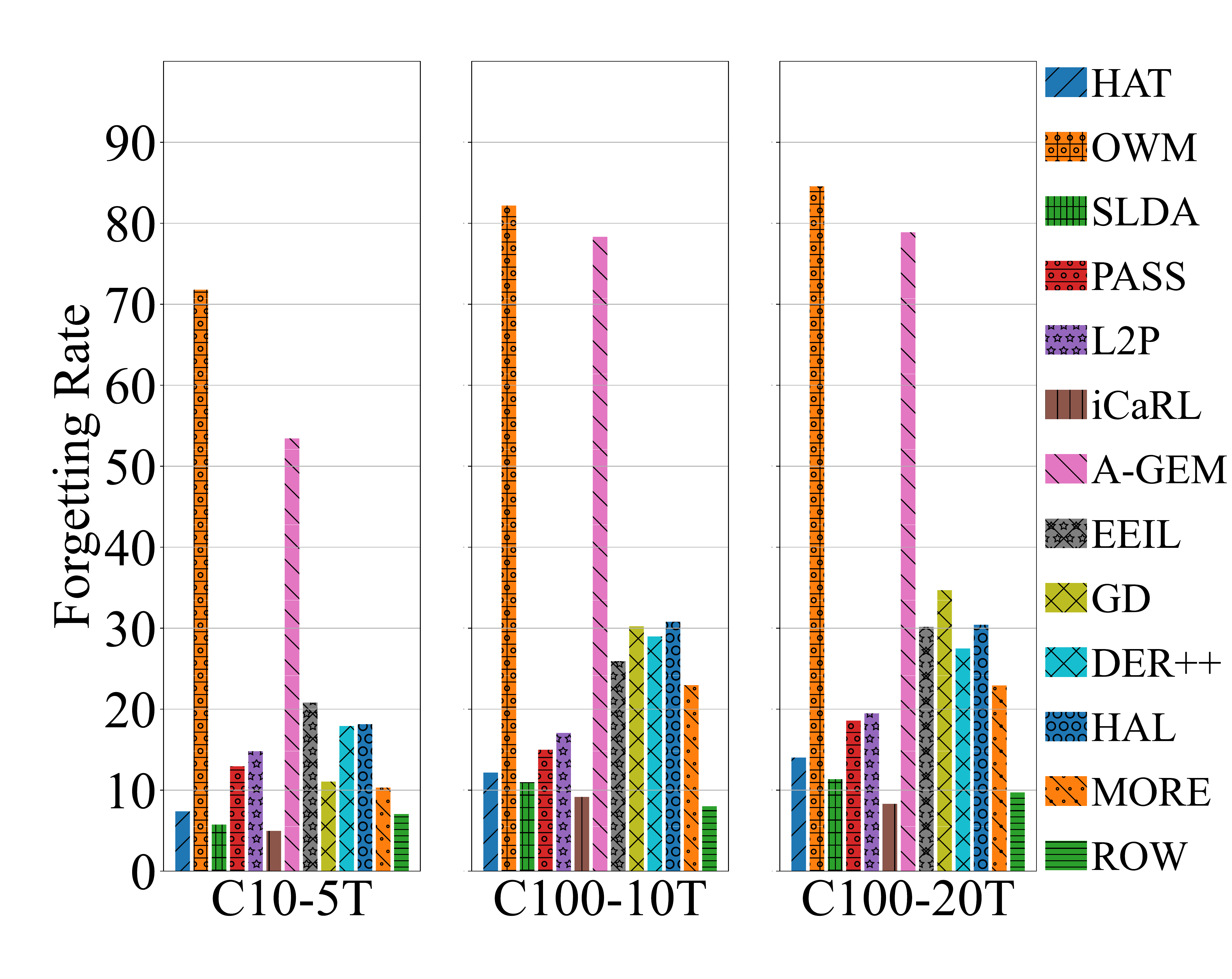}
    \vspace{-0.15in}
    \caption{Average forgetting rate (CIL). The lower the rate, the better the method is.}
    \label{fig:forgetting}
    \vspace{-2mm}
\end{figure}

\textbf{Average Forgetting Rate.} We compare the forgetting rate of each system after learning the last task in Fig.~\ref{fig:forgetting}. The forgetting rates of the proposed method ROW are 7.05, 7.99, and 9.72 on C10-5T, C100-10T and C100-20T, respectively. iCaRL forgets less than ours on C10-5T and C100-20T as it achieves 4.95 and 8.31, respectively. However, iCaRL was not able to adapt to new knowledge effectively as its accuracies are much lower than ROW on the same experiments. The forgetting rate of SLDA on C10-5T 5.74, but a similar observation can be made as iCaRL. 
The average accuracy over the 5 experiments of ROW is 73.72 while that of iCaRL and SLDA are only 66.12 and 68.02, respectively. According to the forgetting rates, the best baseline (MORE) adapts to the new knowledge well, but it was not able to retain the knowledge as effectively as ROW. Its forgetting rates are 10.30, 22.96, and 22.90 on C10-5T, C100-10T, and C100-20T, respectively, and are much larger than ours. This results in lower performance of MORE than ROW.

It is \textbf{important to note} that our system actually has no forgetting due to the CF prevention by HAT. 
The `\textit{forgetting}' occurs not because it forgets the task knowledge, but because the classification becomes harder with more classes.

\begin{table}[t]
\caption{
The classification accuracy of the replay-based baselines and our method ROW with smaller memory buffer sizes. The buffer sizes are reduced by half and the new sizes are: 100 for CIFAR10 and 1,000 for CIFAR100 and Tiny-ImageNet.
Numbers in bold are the best results in each column.
}
\vspace{0.1in}
\centering
\resizebox{1.0\columnwidth}{!}{
\begin{tabular}{l c c c c c c}
\toprule
\multirow{1}{*}{Method}  & \multicolumn{1}{c}{C10-5T}  &  \multicolumn{1}{c}{C100-10T} &  \multicolumn{1}{c}{C100-20T} &  \multicolumn{1}{c}{T-5T} & \multicolumn{1}{c}{T-10T} & \multicolumn{1}{c}{Avg.}\\
\midrule
iCaRL         & 86.08\scalebox{0.5}{$\pm$1.19}  &  66.96\scalebox{0.5}{$\pm$2.08} & 68.16\scalebox{0.5}{$\pm$0.71} & 47.27\scalebox{0.5}{$\pm$3.22} & 49.51\scalebox{0.5}{$\pm$1.87} & 63.60 \\ 
A-GEM         & 56.64\scalebox{0.5}{$\pm$4.29} & 23.18\scalebox{0.5}{$\pm$2.54} & 20.76\scalebox{0.5}{$\pm$2.88} & 31.44\scalebox{0.5}{$\pm$3.84} & 23.73\scalebox{0.5}{$\pm$6.27} & 31.15 \\ 
EEIL & 77.44\scalebox{0.5}{$\pm$3.04} & 62.95\scalebox{0.5}{$\pm$0.68} & 57.86\scalebox{0.5}{$\pm$0.74} & 48.36\scalebox{0.5}{$\pm$1.38} & 44.59\scalebox{0.5}{$\pm$1.72} & 58.24 \\
GD & 85.96\scalebox{0.5}{$\pm$1.64} & 57.17\scalebox{0.5}{$\pm$1.06} & 50.30\scalebox{0.5}{$\pm$0.58} & 46.09\scalebox{0.5}{$\pm$1.77} & 32.41\scalebox{0.5}{$\pm$2.75} & 54.39 \\
DER++         & 80.09\scalebox{0.5}{$\pm$3.00} & 64.89\scalebox{0.5}{$\pm$2.48} & 65.84\scalebox{0.5}{$\pm$1.46} & 50.74\scalebox{0.5}{$\pm$2.41} & 49.24\scalebox{0.5}{$\pm$5.01} & 62.16 \\
HAL           & 79.16\scalebox{0.5}{$\pm$4.56} & 62.65\scalebox{0.5}{$\pm$0.83} & 63.96\scalebox{0.5}{$\pm$1.49} & 48.17\scalebox{0.5}{$\pm$2.94} & 47.11\scalebox{0.5}{$\pm$6.00} & 60.21 \\
MORE  & 88.13\scalebox{0.5}{$\pm$1.16} & 71.69\scalebox{0.5}{$\pm$0.11} & 71.29\scalebox{0.5}{$\pm$0.55} & 64.17\scalebox{0.5}{$\pm$0.77} & 61.90\scalebox{0.5}{$\pm$0.90} & 71.44 \\ 
\hline
ROW & \textbf{89.70}\scalebox{0.5}{$\pm$1.54} & \textbf{73.63}\scalebox{0.5}{$\pm$0.12} & \textbf{71.86}\scalebox{0.5}{$\pm$0.07} & \textbf{65.42}\scalebox{0.5}{$\pm$0.55} & \textbf{62.87}\scalebox{0.5}{$\pm$0.53} & \textbf{72.70} \\
\bottomrule
\end{tabular}
}
\label{Tab:smaller_memory}
\vspace{-2mm}
\end{table}

\subsection{Ablation Experiments} \label{sec:ablation}
\begin{table}
\caption{
Performance gains with the proposed techniques. The method -WP indicates removing WP head and using only OOD head obtained in step 1). The method -MD indicates removing the distance-based coefficient.
}
\vspace{0.1in}
\centering
\resizebox{2.5in}{!}{
\begin{tabular}{l c c c}
\toprule
{} & \multicolumn{1}{c}{C10-5T} & \multicolumn{1}{c}{C100-10T} & \multicolumn{1}{c}{C100-20T}\\
\midrule
\multicolumn{1}{l}{ROW} & 90.97\scalebox{0.5}{$\pm$0.19} & 74.72\scalebox{0.5}{$\pm$0.48} & 74.60\scalebox{0.5}{$\pm$0.12} \\
\multicolumn{1}{l}{ROW (-WP)} & 88.50\scalebox{0.5}{$\pm$1.32} & 72.29\scalebox{0.5}{$\pm$0.90} & 71.97\scalebox{0.5}{$\pm$0.77}  \\
\multicolumn{1}{l}{ROW (-WP-MD)} & 84.06\scalebox{0.5}{$\pm$3.38} & 67.53\scalebox{0.5}{$\pm$1.73} & 65.85\scalebox{0.5}{$\pm$0.95}  \\
\bottomrule
\end{tabular}
}
\label{tab:ablation}
\vspace{-2mm}
\end{table}

We conduct an ablation study to measure the performance after each component is removed from ROW. We consider removing two components: WP head and the distance-based coefficient (MD) in Sec.~\ref{sec.dist_coef}.
The method without WP head (ROW (-WP)) simply uses the OOD head obtained from step 1) with Eq.~\ref{eq.feature_train}. This method makes the final prediction by taking $\arg\max$ over the concatenated logit values without the OOD label from each task network (i.e. OOD head). 

Tab.~\ref{tab:ablation} shows the average classification accuracy. 
The model after removing WP also works greatly as it already outperforms most of the baselines on C10-5T and outperforms the baselines on C100-10T and 20T. In other words, using OOD head constructed following the theoretical framework is effective.
The model is still functional after removing both components (WP and the distance-based coefficient by MD) as shown in the last row of the table (ROW (-WP-MD)).

\section{Conclusion}
To the best of our knowledge, there is still no reported study of learnability of class incremental learning (CIL). This paper performed such a study and showed that the CIL is learnable with some practically reasonable assumptions. A new CIL algorithm (called ROW) was also proposed based on the theory. Experimental results demonstrated that ROW outperforms strong baselines.

\section*{Acknowledgements}
The work of Gyuhak Kim and Bing Liu was supported in part by a research contract from KDDI and three NSF grants (IIS-1910424, IIS-1838770, and CNS-2225427).

\bibliography{icml2023}
\bibliographystyle{icml2023}

\newpage
\appendix
\onecolumn

\textbf{APPENDIX}

\section{Proof}
\label{sec.proof}

\begin{proof}[Proof of \textbf{Theorem} \ref{thm:def2_not_imply_def1}]
\label{prf:def2_not_imply_def1}
Denote the algorithm that satisfies \textbf{Definition} \ref{def:part_obs_closed_learn} as $\textbf{A}_k^r$. 
Given any $\epsilon_n$ and $S \sim D_{[1:k]}$, denote $h_k = \textbf{A}_k^r (S)$. 
Based on $h_k$, we construct a $\Tilde{h}_k$ that satisfies \textbf{Definition} \ref{def:part_obs_closed_learn} but doesn't satisfy \textbf{Definition} \ref{def:full_obs_closed_learn}. 

We define $\mathcal{H}$ has the \textit{capacity} to learn more than one task as follows.
For any $h_0 \in \mathcal{H}$ that could only make correct predictions on a single task, but wrong predictions on all the other tasks, there exists $\delta > 0$ s.t. 
$$
\inf_{h \in \mathcal{H}} \mathbf{R}_{D_{[1:k]}} (h) < \mathbf{R}_{D_{[1:k]}} (h_0) - \delta.
$$

Denote $h_k = {\arg\max}_{i,j} \{\dots, z_k^{i,j},\dots\}$, where $z_k^{i,j} $ is the score function of the $j$-th class of the $i$-th task. 
Let $\sigma(z) = 1 / (1 + e^{-z})$ to be the sigmoid function. 
Define $$\Tilde{z}_k^{i,j}(x) = \sigma(z_k^{i,j}(x)) + i$$ and $$\Tilde{h}_k = \arg\max_{i,j} \{\dots, \Tilde{z}_k^{i,j},\dots\}.$$ 

\textbf{(i)} Since $0 < \sigma < 1$, we have $\Tilde{z}_k^{i,j} < \Tilde{z}_k^{i',j},\, \forall \, i < i'$. 
Therefore, 
$$\Tilde{h}_k ={\arg\max}_{i,j} \{\dots, \Tilde{z}_k^{i,j},\dots\}= {\arg\max}_{i=k,j} \{\dots, \Tilde{z}_k^{i,j},\dots\}. $$ 
Since $\sigma$ is monotonic increasing, we have
$$
\begin{aligned}
    \mathbf{R}_{D_{k}} (h_k) 
    &= \mathbb{E}_{(x, y)\sim D_k}[l({\arg\max}_{i=k, j} \{\dots, z_k^{i,j}(x),\dots\}, y)] \\
    &= \mathbb{E}_{(x, y)\sim D_k}[l({\arg\max}_{i=k, j} \{\dots, \Tilde{z}_k^{i,j}(x),\dots\}, y)] 
    = \mathbf{R}_{D_{k}} (\Tilde{h}_k). 
\end{aligned}
$$
Plugging $\mathbf{R}_{D_{k}} (h_k) = \mathbf{R}_{D_{k}} (\Tilde{h}_k) $ into \textbf{Definition} \ref{def:part_obs_closed_learn}, we have 
$$\max_{k=1, \dots, T} \mathbb{E}_{S \sim D_{[1:k]}} [\mathbf{R}_{D_{k}} (\Tilde{h}_k) - \inf_{h \in \mathcal{H}} \mathbf{R}_{D_{k}} (h)] < \epsilon_n .$$
Therefore, $\Tilde{h}_k$ also satisfies \textbf{Definition} \ref{def:part_obs_closed_learn}.

\textbf{(ii)} Since $\Tilde{h}_k$ always predicts the class of the $k$-th task, all predicted labels of samples from $D_{[1:k-1]}$ are wrong. 
Therefore, we have 
$$
    \max_{k=1, \dots, T} \mathbb{E}_{S \sim D_{[1:k]}} [\mathbf{R}_{D_{[1:k]}} (\Tilde{h}_k) - \inf_{h \in \mathcal{H}} \mathbf{R}_{D_{[1:k]}} (h)] > \delta > 0.
$$
Because $\delta$ is a constant that is irrelevant to $S \sim D_{[1:k]}$,
it cannot be reduced by increasing samples.  
Therefore, $\Tilde{h}_k$ doesn't satisfy \textbf{Definition} \ref{def:full_obs_closed_learn}.
\end{proof}

\clearpage

\begin{proof}[Proof of \textbf{Theorem} \ref{thm:def3_imply_def1}]
\label{prf:def3_imply_def1}
Denote the algorithm that satisfies \textbf{Definition} \ref{def:full_obs_open_learn} as $\textbf{A}_k$. 
Define $h_k = \textbf{A}_k (S)$. 
Let $\alpha_1 = \dots = \alpha_k = 0$, then we have $D_{[1:k]}^{\alpha_{[1: k]}} = D_{[1:k]}$.
It's obvious that $h_k$ satisfies \textbf{Definition} \ref{def:full_obs_closed_learn} because  
$$\mathbf{R}_{D_{[1:k]}} (h_k) - \inf_{h \in \mathcal{H}} \mathbf{R}_{D_{[1:k]}} (h) = \mathbf{R}_{D_{[1:k]}^{\alpha_{[1: k]}}} (h_k) - \inf_{h \in \mathcal{H}} \mathbf{R}_{D_{[1:k]}^{\alpha_{[1: k]}}} (h).$$
\end{proof}

\clearpage

\begin{proof}[Proof of \textbf{Theorem} \ref{thm:def4_imply_def3}]
\label{prf:def4_imply_def3}
Denote the algorithm that satisfies \textbf{Definition} \ref{def:part_obs_open_learn} as $\textbf{A}_k^r$. 
Given any $\epsilon_n$ and $S \sim D_{[1:k]}$, denote $h_k = \textbf{A}_k^r (S)$. 
Based on $h_k$, we construct a $\Tilde{h}_k$ that satisfies \textbf{Definition} \ref{def:full_obs_open_learn}.

For simplicity, we denote $\pi_{[k: k']} = \sum_{i=k}^{k'} \pi_i$.

For any $O_{(X_1, Y_1)}, \dots, O_{(X_T, Y_T)} \in \mathcal{D}$ and any $\alpha_1, \dots, \alpha_T \in [0, 1)$, let 
$$\alpha'_k = 1 - \frac{\pi_k}{\pi_{[k:T]}} (1 - \alpha_k)
$$
and 
$$O'_{(X_k, Y_k)} = \frac{\pi_k}{\pi_{[k:T]}} \alpha_k O_{(X_k, Y_k)} + \sum_{i=k+1}^{T} \frac{\pi_i}{\pi_{[k:T]}} [(1 - \alpha_i) D_i + \alpha_i O_{(X_i, Y_i)}].
$$
Defining
$$D_k^{\alpha'_k} \overset{def}{=} (1 - \alpha'_k) D_k + \alpha'_k O'_{(X_k, Y_k)}$$
and plugging $\alpha'_k$ and $O'_{(X_k, Y_k)}$ into \textbf{Definition} \ref{def:part_obs_open_learn}, we have
$$\max_{k=1, \dots, T} \mathbb{E}_{S \sim D_{[1:k]}} [\mathbf{R}_{D_{k}^{\alpha'_k}} (h_k) - \inf_{h \in \mathcal{H}} \mathbf{R}_{D_{k}^{\alpha'_k}} (h)] < \epsilon_n. $$

Denote $h_k = {\arg\max} \{\dots, z_k^{j},\dots; z_k^o\}$, where $z_k^{j} $ is the score function of the $j$-th class of the $k$-th task, and $z_k^o$ is the score function of the OOD class of the $k$-th task. 
Denote the label of the $j$-th class of the $i$-th task as $y^{i, j}$. 
Denote the label of OOD class of the $i$-th task as $y^{i, o}$. 
We define $$\Tilde{h}_k(x) = \left\{
\begin{aligned}
    &y^{i, j}\ \text{if}\  h_{i'}(x) = y^{i', o}, \forall\, i' < i,\ \text{and}\ h_{i}(x) = y^{i, j}; \\
    &y^{k, o}\ \text{if}\  h_{i'}(x) = y^{i', o}, \forall\, i' \leq k. 
\end{aligned}
\right.$$

By definition of $\Tilde{h}_k$, when $\Tilde{h}_k$ makes a wrong prediction, there exists a $h_{i'}, i' \leq k$ that makes a mistake. 
We assume that the loss function satisfies the following inequality 
\begin{equation}
l(\Tilde{h}_k(x), y) \leq \left\{
\begin{aligned}
    &\sum_{t=1}^{i-1} l(h_t(x), y^{t, o}) + l(h_i(x), y^{i, j}), \\
    &\ \ \ \ \ \ \ \ \ \ \ \ \ \ \ \ \text{if}\  h_{i'}(x) = y^{i', o}, \forall\, i' < i,\ \text{and}\ h_{i}(x) = y^{i, j};\\
    &\sum_{t=1}^{k} l(h_t(x), y^{t, o}), \\
    &\ \ \ \ \ \ \ \ \ \ \ \ \ \ \ \ \text{if}\  h_{i'}(x) = y^{i', o}, \forall\, i' \leq k.
\end{aligned}
\right.
\label{eq:loss_assump}
\end{equation}
Then we can decompose the risk function of $\Tilde{h}_k$, 
\begin{equation}
\begin{aligned}
    \mathbf{R}_{D_{[1:k]}^{\alpha_{[1: k]}}} (\Tilde{h}_k)
    &= \mathbb{E}_{(x, y) \sim D_{[1:k]}^{\alpha_{[1: k]}}} l(\Tilde{h}_k (x), y) \\
    &= \sum_{i=1}^{k} \mathbb{E}_{(x, y) \sim D_{[1:k]}^{\alpha_{[1: k]}}} l(\Tilde{h}_k (x), y) 1_{(x, y) \sim D_{i}} \\
    &\ \ \ + \sum_{i=1}^{k} \mathbb{E}_{(x, y) \sim D_{[1:k]}^{\alpha_{[1: k]}}} l(\Tilde{h}_k (x), y) 1_{(x, y) \sim O_{i}} \\
    &\leq \sum_{i=1}^{k} \mathbb{E}_{(x, y) \sim D_{[1:k]}^{\alpha_{[1: k]}}} [\sum_{t=1}^{i-1} l(h_t (x), y^{t, o}) + l(h_i (x), y^{i, j})] 1_{(x, y) \sim D_{i}} \\
    &\ \ \ + \sum_{i=1}^{k} \mathbb{E}_{(x, y) \sim D_{[1:k]}^{\alpha_{[1: k]}}} [\sum_{t=1}^{i} l(h_t (x), y^{t, o}) ]  1_{(x, y) \sim O_{i}}.
\end{aligned}
\label{eq:risk_decomp}
\end{equation}

With the fact that for any density function $p(x)$ defined on $A$ and any $B \subset A$, 
$$
\int_{A} \frac{p(x)}{\int_A p(x) dx} f(x) 1_{x \in B} dx 
= \frac{\int_B p(x) dx}{\int_A p(x) dx} \int_{B} \frac{p(x)}{\int_B p(x) dx} f(x) dx,
$$
by definition of $\alpha'_k$ and $O'_{(X_k, Y_k)}$, we have that for $t < i \leq k$, 
$$
\begin{aligned}
    &\ \ \ \ \mathbb{E}_{(x, y) \sim D_{[1:k]}^{\alpha_{[1: k]}}} l(h_t (x), y^{t, o}) 1_{(x, y) \sim D_{i}} \\
    &= \frac{(1 - \alpha_i)\pi_i}{\pi_{[1:k]}} \mathbb{E}_{(x, y) \sim D_{i}} l(h_t (x), y^{t, o}) \\
    &= \frac{(1 - \alpha_i)\pi_i}{\pi_{[1:k]}} \frac{\pi_{[t:T]}}{(1-\alpha_i)\pi_i} \mathbb{E}_{(x, y) \sim D_{t}^{\alpha'_t}} l(h_t (x), y^{t, o}) 1_{(x, y) \sim D_{i}} \\
    &= \frac{\pi_{[t:T]}}{\pi_{[1:k]}} \mathbb{E}_{(x, y) \sim D_{t}^{\alpha'_t}} l(h_t (x), y^{t, o}) 1_{(x, y) \sim D_{i}}, \\
\end{aligned}
$$
$$
\begin{aligned}
    &\ \ \ \ \mathbb{E}_{(x, y) \sim D_{[1:k]}^{\alpha_{[1: k]}}} l(h_i (x), y^{i, j}) 1_{(x, y) \sim D_{i}} \\
    &= \frac{(1 - \alpha_i)\pi_i}{\pi_{[1:k]}} \mathbb{E}_{(x, y) \sim D_{i}} l(h_i (x), y^{i, j}) \\
    &= \frac{(1 - \alpha_i)\pi_i}{\pi_{[1:k]}} \frac{\pi_{[i:T]}}{(1-\alpha_i)\pi_i} \mathbb{E}_{(x, y) \sim D_{i}^{\alpha'_i}} l(h_i (x), y^{i, j}) 1_{(x, y) \sim D_{i}} \\
    &= \frac{\pi_{[i:T]}}{\pi_{[1:k]}} \mathbb{E}_{(x, y) \sim D_{i}^{\alpha'_i}} l(h_i (x), y^{i, j}) 1_{(x, y) \sim D_{i}}, \\
\end{aligned}
$$
and for $t \leq i \leq k$, 
$$
\begin{aligned}
    &\ \ \ \ \mathbb{E}_{(x, y) \sim D_{[1:k]}^{\alpha_{[1: k]}}} l(h_t (x), y^{t, o}) 1_{(x, y) \sim O_{i}} \\
    &=\frac{\alpha_i \pi_i}{\pi_{[1:k]}} \mathbb{E}_{(x, y) \sim O_{i}} l(h_t (x), y^{t, o}) \\
    &= \frac{\alpha_i \pi_i}{\pi_{[1:k]}} \frac{\pi_{[t:T]}}{\alpha_i \pi_i} \mathbb{E}_{(x, y) \sim D_{t}^{\alpha'_t}} l(h_t (x), y^{t, o}) 1_{(x, y) \sim O_{i}} \\
    &= \frac{\pi_{[t:T]}}{\pi_{[1:k]}} \mathbb{E}_{(x, y) \sim D_{t}^{\alpha'_t}} l(h_t (x), y^{t, o}) 1_{(x, y) \sim O_{i}}. \\
\end{aligned}
$$
Plugging the above three equations into Eq.~\ref{eq:risk_decomp}, we have 
$$
\begin{aligned}
    \mathbf{R}_{D_{[1:k]}^{\alpha_k}} (\Tilde{h}_k)
    &\leq \sum_{i=1}^{k} \sum_{t=1}^{i-1} \frac{\pi_{[t:T]}}{\pi_{[1:k]}} \mathbb{E}_{(x, y) \sim D_{t}^{\alpha'_t}} l(h_t (x), y^{t, o}) 1_{(x, y) \sim D_{i}} \\ 
    &\ \ \ + \sum_{i=1}^{k} \frac{\pi_{[i:T]}}{\pi_{[1:k]}} \mathbb{E}_{(x, y) \sim D_{i}^{\alpha'_i}} l(h_i (x), y^{i, j}) 1_{(x, y) \sim D_{i}} \\
    &\ \ \ + \sum_{i=1}^k \sum_{t=1}^{i} \frac{\pi_{[t:T]}}{\pi_{[1:k]}} \mathbb{E}_{(x, y) \sim D_{t}^{\alpha'_t}} l(h_t (x), y^{t, o}) 1_{(x, y) \sim O_{i}} \\
    &= \sum_{t=1}^{k} \sum_{i=t+1}^{k} \frac{\pi_{[t:T]}}{\pi_{[1:k]}} \mathbb{E}_{(x, y) \sim D_{t}^{\alpha'_t}} l(h_t (x), y^{t, o}) 1_{(x, y) \sim D_{i}} \\ 
    &\ \ \ + \sum_{t=1}^{k} \frac{\pi_{[t:T]}}{\pi_{[1:k]}} \mathbb{E}_{(x, y) \sim D_{t}^{\alpha'_t}} l(h_t (x), y^{t, j}) 1_{(x, y) \sim D_{t}} \\
    &\ \ \ + \sum_{t=1}^k \sum_{i=t}^{k} \frac{\pi_{[t:T]}}{\pi_{[1:k]}} \mathbb{E}_{(x, y) \sim D_{t}^{\alpha'_t}} l(h_t (x), y^{t, o}) 1_{(x, y) \sim O_{i}} \\
    &= \sum_{t=1}^k \frac{\pi_{[t:T]}}{\pi_{[1:k]}} \mathbb{E}_{(x, y) \sim D_{t}^{\alpha'_t}} l(h_t (x), y^{t, *}) 1_{(x, y) \sim \cup_{i=t}^k (D_{i} \cup O_{i})} \\
    &= \sum_{t=1}^k \frac{\pi_{[t:T]}}{\pi_{[1:k]}} \mathbf{R}_{D_t^{\alpha'_t}}(h_t). 
\end{aligned}
$$

By assumption that $\inf_{h \in \mathcal{H}} \mathbf{R}_{D_{[1:k]}^{\alpha_{[1: k]}}} (h) = 0$, it's obvious that $\inf_{h \in \mathcal{H}} \mathbf{R}_{D_{k}^{\alpha'_{k}}} (h) = 0$,
which means that 
$$
\max_{k=1, \dots, T} \mathbb{E}_{S \sim D_{[1:k]}} [\mathbf{R}_{D_{k}^{\alpha'_k}} (h_k)] < \epsilon_n.
$$
Therefore, we have
$$
\begin{aligned}
    &\ \ \ \max_{k=1, \dots, T} \mathbb{E}_{S \sim D_{[1:k]}} [\mathbf{R}_{D_{[1:k]}^{\alpha_{[1: k]}}} (\Tilde{h}_k) - \inf_{h \in \mathcal{H}} \mathbf{R}_{D_{[1:k]}^{\alpha_{[1: k]}}} (h)] \\
    &\leq \max_{k=1, \dots, T} \mathbb{E}_{S \sim D_{[1:k]}} [\sum_{t=1}^k \frac{\pi_{[t:T]}}{\pi_{[1:k]}} \mathbf{R}_{D_t^{\alpha'_t}}(h_t)]  \\
    &< \epsilon_n \cdot \max_{k=1, \dots, T} \sum_{t=1}^k \frac{\pi_{[t:T]}}{\pi_{[1:k]}}.
\end{aligned}
$$

\end{proof}

\clearpage

\begin{proof}[Proof of \textbf{Corollary} \ref{thm:def4_imply_def3_replay}]

Denote the algorithm that satisfies \textbf{Definition} \ref{def:part_obs_open_learn} as $\textbf{A}_k^r$. 
Given any $\epsilon_n$ and $S \sim D_{[1:k]}$, denote $h_k = \textbf{A}_k^r (S)$. 
Based on $h_k$, we construct a $\Tilde{h}_k$ that satisfies \textbf{Definition} \ref{def:full_obs_open_learn}.

For simplicity, we denote $\pi_{[k: k']} = \sum_{i=k}^{k'} \pi_i$.

Different from proof of \textbf{Theorem} \ref{thm:def4_imply_def3}, when we can acquire replay data and therefore treat them as OOD data, we let 
$$
\alpha'_k = 1 - \frac{\pi_k}{\pi_{[1:T]}} (1 - \alpha_k)
$$
and
$$
O'_{(X_k, Y_k)} = \frac{\pi_k}{\pi_{[1:T]}} \alpha_k O_{(X_k, Y_k)} + \sum_{i \neq k} \frac{\pi_i}{\pi_{[1:T]}} [(1 - \alpha_i) D_i + \alpha_i O_{(X_i, Y_i)}].
$$
Defining
$$D_k^{\alpha'_k} \overset{def}{=} (1 - \alpha'_k) D_k + \alpha'_k O'_{(X_k, Y_k)}$$
and plugging $\alpha'_k$ and $O'_{(X_k, Y_k)}$ into \textbf{Definition} \ref{def:part_obs_open_learn}, we have
$$\max_{k=1, \dots, T} \mathbb{E}_{S \sim D_{[1:k]}} [\mathbf{R}_{D_{k}^{\alpha'_k}} (h_k) - \inf_{h \in \mathcal{H}} \mathbf{R}_{D_{k}^{\alpha'_k}} (h)] < \epsilon_n. $$

Denote $h_k = {\arg\max} \{\dots, z_k^{j},\dots; z_k^o\}$, where $z_k^{j} $ is the score function of the $j$-th class of the $k$-th task, and $z_k^o$ is the score function of the OOD class of the $k$-th task. 
Denote the label of the $j$-th class of the $i$-th task as $y^{i, j}$. 
Denote the label of OOD class the the $i$-th task as $y^{i, o}$. 
We define $$\Tilde{h}_k(x) = \left\{
\begin{aligned}
    &y^{i, j}\ \text{if}\  h_{i}(x) = y^{i, j}, \exists\, i \leq k; \\
    &y^{k, o}\ \text{if}\  h_{i'}(x) = y^{i', o}, \forall\, i' \leq k. 
\end{aligned}
\right.$$

It's ideal that $h_{i'}(x) = y^{i', o}, \forall\, i' \neq i$ when $h_{i}(x) = y^{i, j}, \exists\, i \leq k$. But when $\Tilde{h}_k$ makes a wrong prediction, there exists a $h_{i'}, i' \leq i$ that makes a mistake. 
We assume that the loss function satisfies the following inequality 
\begin{equation}
l(\Tilde{h}_k(x), y) \leq \left\{
\begin{aligned}
    &\sum_{t \neq i} l(h_t(x), y^{t, o}) + l(h_i(x), y^{i, j}), \\
    &\ \ \ \ \ \ \ \ \ \ \ \ \ \ \ \ \text{if}\  h_{i}(x) = y^{i, j}, \exists\, i \leq k;\\
    &\sum_{t=1}^{k} l(h_t(x), y^{t, o}), \\
    &\ \ \ \ \ \ \ \ \ \ \ \ \ \ \ \ \text{if}\  h_{i'}(x) = y^{i', o}, \forall\, i' \leq k.
\end{aligned}
\right.
\label{eq:loss_assump_replay}
\end{equation}

All the same as the proof of \textbf{Theorem} \ref{thm:def4_imply_def3}, we have 
$$
\begin{aligned}
    \mathbf{R}_{D_{[1:k]}^{\alpha_{[1: k]}}} (\Tilde{h}_k)
    &= \mathbb{E}_{(x, y) \sim D_{[1:k]}^{\alpha_{[1: k]}}} l(\Tilde{h}_k (x), y) \\
    &\leq \sum_{i=1}^{k} \mathbb{E}_{(x, y) \sim D_{[1:k]}^{\alpha_{[1: k]}}} [\sum_{t \neq i} l(h_t (x), y^{t, o}) + l(h_i (x), y^{i, j})] 1_{(x, y) \sim D_{i}} \\
    &\ \ \ + \sum_{i=1}^{k} \mathbb{E}_{(x, y) \sim D_{[1:k]}^{\alpha_{[1: k]}}} [\sum_{t=1}^{i} l(h_t (x), y^{t, o}) ]  1_{(x, y) \sim O_{i}} \\
    &\leq \epsilon_n \cdot \max_{k=1, \dots, T} \sum_{t=1}^k \frac{\pi_{[1:T]}}{\pi_{[1:k]}} 
    = \epsilon_n \cdot \max_{k=1, \dots, T} \frac{k \pi_{[1:T]}}{\pi_{[1:k]}}.
\end{aligned}
\label{eq:risk_decomp_2}
$$

\end{proof}

\clearpage

\section{Hard Attention to the Task (HAT)} \label{sec:hat}
In training the network $h_{k} \circ f_k$ using the data of task $k$ and the generated pseudo feature vectors, we employ the hard attention mask~\cite{Serra2018overcoming} to prevent forgetting in the feature extractor. 

The hard attention mask $a_{l}^{k}$ is a trainable pseudo binary 0-1 vector at each layer $l$ of task $k$. It is element-wise multiplied to the output of the layer as $a_{l}^{k} \otimes h_{l}$ and blocks (for value of $0$) or unblocks (for value of 1) the information flow from neurons of adjacent layers. Neurons with value 1 are important for the task and thus need to be protected while neurons with value 0 are not necessary for the task and can be freely modifed without affecting other tasks.

More specifically, we modify the gradients of parameters that are important in performing the previous tasks $(1, \cdots, k-1)$ during training task $k$ so the important parameters for previous tasks are unaffected. The gradient of parameter $w_{ij, l}$ at $i$th row and $j$th column of layer $l$ is modified as
\begin{align}
    \nabla w_{ij, l}' = \left( 1 - \min\left( a_{i,l}^{<k}, a_{j, l-1}^{<k} \right) \right) \nabla w_{ij, l}, \label{eq:grad_mod}
\end{align}
where $a_{i,l}^{<k}$ is an accumulated attentions over previous tasks and is 1 if the hard attention of neuron $i$ at layer $l$ is ever used by any previous task $< k$ (see~\cite{Serra2018overcoming} for details).

To encourage parameter sharing and sparsity in the number of activated masks, a regularization is introduced as $\mathcal{L}_r = \sum_{l,i} a_{i,l}^{k}( 1 - a_{i,l}^{<k} ) \big/ \sum_{l,i} (1 - a_{i, l}^{<k} )$.
The final objective to train a comprehensive task network without forgetting is
\begin{align}
    \mathcal{L} = \mathcal{L}_{ood} + \mathcal{L}_{r}, \label{eq:final_obj}
\end{align}
where $\mathcal{L}_{ood}$ is the cross-entropy loss in Eq.~\ref{eq.feature_train}.

\section{Hyper-Parameters} \label{sec.hyper_param}
For all the experiments, we use SGD with momentum value of 0.9 with batch size of 64. For C10-5T, we use learning rate 0.005 and train for 20 epochs. For C100-10T and 20T, we train for 40 epochs with learning rate 0.001 and 0.005 for 10T and 20T, respectively. For T-5T and 10T, we use the same learning rate 0.005, but train for 15 and 10 epochs for 5T and 10T, respectively. For fine-tuning WP and OOD head, we use batch size of 32 and use the same learning rate used for training the feature extractor. For fine-tuning WP and TP, we train for 5 epochs and 10 epochs, respectively.

\section{Required Memory}
\begin{table*}[h]
\caption{
The size of the model (in entries) required for each method without the memory buffer.
}
\vspace{0.1in}
\centering
\resizebox{0.5\columnwidth}{!}{
\begin{tabular}{l c c c c c}
\toprule
\multirow{1}{*}{Method}  & \multicolumn{1}{c}{C10-5T}  &  \multicolumn{1}{c}{C100-10T} &  \multicolumn{1}{c}{C100-20T} &  \multicolumn{1}{c}{T-5T} & \multicolumn{1}{c}{T-10T} \\
\midrule
HAT & 23.0M & 24.7M & 25.4M & 24.6M & 25.1M \\
OWM & 26.6M & 28.1M & 28.1M & 28.2M & 28.2M \\
SLDA & 21.6M & 21.6M & 21.6M & 21.7M & 21.7M \\
PASS & 22.9M & 24.2M & 24.2M & 24.3M & 24.4M \\
L2P & 21.7M & 21.7M & 21.7M & 21.8M & 21.8M \\
iCaRL & 22.9M & 24.1M & 24.1M & 24.1M & 24.1M \\
A-GEM & 26.5M & 31.4M & 31.4M & 31.5M & 31.5M \\
EEIL & 22.9M & 24.1M & 24.1M & 24.1M & 24.1M \\
GD & 22.9M & 24.1M & 24.1M & 24.1M & 24.1M \\
DER++ & 22.9M & 24.1M & 24.1M & 24.1M & 24.1M \\
HAL & 22.9M & 24.1M & 24.1M & 24.1M & 24.1M \\
MORE & 23.7M & 25.9M & 27.7M & 25.1M & 25.9M  \\
ROW & 23.7M & 26.0M & 27.8M & 25.2M & 26.0M \\
\bottomrule
\end{tabular}
}
\label{Tab:memory}
 \vspace{-3mm}
\end{table*}
We report the network sizes of the systems after learning the last task. We use an `entry' to denote a parameter or a value required to learn and to do inference for a task. 

All the systems except SLDA and L2P use the feature extractor DeiT-S/16~\citep{touvron2021training_deit} and adapter modules. The transformer consumes 21.6 millions (M) entries and the adapters take 1.2M and 2.4M entries for CIFAR10 and the other datasets. SLDA fine-tunes only the classifier on top of the fixed pre-trained feature extractor as it does not have a protection mechanism. L2P uses a prompt pool with 23k entries. Since each method requires method-specific elements (e.g., task embedding for HAT), the number of entries required for each method is different. The number of entries for each model is reported in Tab.~\ref{Tab:memory}.

Our system saves the covariance matrices for computing the distance-based coefficient in Sec.~\ref{sec.dist_coef}. The covariances are saved for each task. Since each covariance is in size 384x384, the total entries for this step are 737.3k, 1.5M, 2.9M, 737.3k, and 1.5M for C10-5T, C100-10T, C100-20T, T-5T, and T-10T, respectively. The numbers are relatively small considering that some of the replay-based methods (e.g., iCaRL, HAL) require a teacher model the same size as the training model for knowledge distillation. More importantly, replay buffer requires the largest memory (e.g., for Tiny-ImageNet, saving 2,000 images of size 64x64x3 requires 24.6M entries). It is highly important that the system is robust to replay buffer size. ROW is shown to remain strong with small replay buffer sizes (see Tab.~\ref{Tab:smaller_memory}).

Our system ROW saves an additional classifier. WP head is of the same shape as the classifier of the standard baselines (e.g., iCaRL or DER++). OOD (or TP) head requires the same memory as WP with additional parameters for OOD class per task. The required memory is small. For instance, for C10-5T, using OOD head only introduces 5,775 additional entries.

{
\section{Societal Impact and Limitation}\label{impact_limit}
We do not see any negative societal impact as we use public domain datasets for the experiments and our algorithm is just like any normal supervised learning in nature. In practical use, if the training data of the application is biased, it could affect the model just like in any other supervised learning. We believe that this can be alleviated by checking any potential bias in the training data. 

The current theoretical study is only applicable to offline CIL. In future work, we will extend our study to online CIL, where the task boundary may be blurry.

\section{Effect on Underrepresented Minorities} \label{minority}
The existence of underrepresented samples in the current task does not affect our theory, but it will affect an actual implementation and give weaker results. The OOD detection method in our paper simply trains the system by considering the samples of the current task as in-distribution and the samples in the replay buffer as OOD of the current task. In case there is a set of underrepresented classes in the current task’s dataset, one can use existing techniques proposed for the sample imbalance problem to alleviate the issue. However, this problem is not just relevant to our proposed method, but relevant to all existing OOD and CL methods, or even supervised learning methods.
}

\end{document}